%% file: AnonymousSubmission2027.tex
\documentclass[letterpaper]{article} 
\usepackage[preprint]{aaai2027}  
\usepackage[hyphens]{url}  
\usepackage{graphicx} 
\usepackage{natbib}  
\usepackage{caption} 
\usepackage{algorithm}
\usepackage{algorithmic}

\usepackage{microtype}
\usepackage{graphicx}
\usepackage{subcaption}
\usepackage{tabularx}
\usepackage{makecell}
\usepackage{array}
\usepackage{multirow}
\usepackage{xcolor}

\usepackage{amsmath}
\usepackage{amssymb}
\usepackage{mathtools}
\usepackage{amsthm}

\usepackage{enumitem}
\input{math.tex}

\usepackage{newfloat}
\usepackage{listings}
\DeclareCaptionStyle{ruled}{labelfont=normalfont,labelsep=colon,strut=off} 
\floatstyle{ruled}
\newfloat{listing}{tb}{lst}{}
\floatname{listing}{Listing}

\usepackage{booktabs}

\title{ABO-Med:Accelerated Bilevel Optimization \\ for Few-Shot Medical Image Classification}

\author{
    Ruoxuan Shi\textsuperscript{\rm 1}\equalcontrib,
    Sheng Yang\textsuperscript{\rm 2}\equalcontrib\corresponding,
    Zhengxing Su\textsuperscript{\rm 1},
    Xiaoyang Hou\textsuperscript{\rm 2},
    Yating Liu\textsuperscript{\rm 3}\corresponding
}

\affiliations{
    \textsuperscript{\rm 1}School of Mathematics and Statistics, Lanzhou University\\
    \textsuperscript{\rm 2}Department of Statistics, University of California, Riverside\\
    \textsuperscript{\rm 3}Department of Oncology, Lanzhou University Second Hospital\\
    shirx2025@lzu.edu.cn, syang362@ucr.edu, suzhx2024@lzu.edu.cn\\
    shyannhou@gmail.com, ery\_liuytery@lzu.edu.cn
}

\begin{document}

\maketitle

\begin{abstract}
In recent years, bilevel optimization has been widely used in a variety of machine learning tasks. However, prior bilevel optimization algorithms generally require the computation of second-order information, which limits their practical scalability. Only recently has a first-order paradigm for bilevel optimization been established, attaining near-optimal theoretical guarantees for solving bilevel optimization problems. In this paper, we propose ABO-Med, a scalable instantiation of this paradigm for few-shot learning, by incorporating it into the model-agnostic meta-learning (MAML) framework and tailoring it to medical image classification. We also introduce Medical Adaptive RandomAugment (MedRAug), a modality-aware augmentation strategy designed for medical images. Theoretically, ABO-Med establishes the optimality of MAML-type meta-learning approaches. Empirically, ABO-Med outperforms prior baselines on several public medical datasets, with gains of 1.99\% to 18.76\%, while MedRAug further improves the average accuracy by 2.20\% to 6.34\%. Additional cross-domain experiments, augmentation ablation studies, backbone ablation studies, and training efficiency analysis further validate the effectiveness and efficiency of the proposed method.
\end{abstract}



\begin{figure*}[t]
    \centering
    \includegraphics[width=0.95\textwidth]{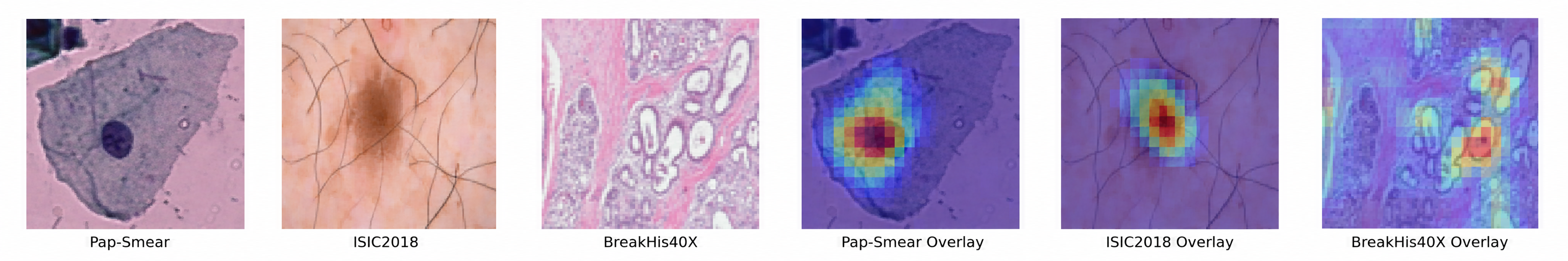}
    \caption{Grad-CAM visualization of the ABO-Med method trained under the 3-way 10-shot setting on Pap-Smear, ISIC2018, and BreakHis40X, using the baseline augmentation strategy. Warmer colors indicate regions that contribute more strongly to the model prediction.}
    \label{fig:gradcam}
\end{figure*}

\section{Introduction}
In medical image analysis and computer-aided diagnosis, acquiring large-scale, high-quality annotated data is often challenging. Medical data annotation typically relies on expert clinicians, making it expensive and time-consuming. As a result, many tasks only have limited labeled samples, which hinders the performance of deep learning models~\cite{litjens2017survey,esteva2019guide,greenspan2016guest} that rely on large-scale training data. To address this issue, few-shot learning~\cite{wang2020generalizing,vinyals2016matching,snell2017prototypical} aims to achieve effective classification or prediction with only a small number of training samples. Among various approaches, meta-learning~\cite{hospedales2021meta,huisman2021survey} aims to learn transferable prior knowledge or learning strategies across a distribution of tasks, so that a model can rapidly adapt to a new task with only a few samples. It generally consists of two stages: meta-training, where transferable knowledge is learned from task distributions, and meta-testing, where the learned knowledge is evaluated on unseen tasks.
Existing meta-learning methods can be broadly categorized into metric-based~\cite{vinyals2016matching}, model-based~\cite{santoro2016meta}, and optimization-based approaches~\cite{finn2017model}. In this paper, we focus on the optimization-based paradigm, as it is naturally aligned with optimization methods in a broad range of machine learning problems.

Model-Agnostic Meta-Learning~(MAML)~\cite{finn2017model} is a prominent optimization-based meta-learning method that learns an initialization for fast adaptation to new tasks. However, differentiating through the inner-loop updates in the outer loop can incur costly Hessian computations and high memory overhead. To improve efficiency, numerous MAML variants~\cite{nichol2018first,rajeswaran2019meta,raghu2019rapid,antoniou2018train,fallah2020convergence,baik2020meta,chayti2024new} have been proposed. Among them, ANIL~(Almost No Inner Loop)~\cite{raghu2019rapid} freezes the feature extractor during task adaptation and updates only the last layer, thereby significantly reducing computation while maintaining competitive performance. Under the ANIL framework, the meta-training stage of learning a good parameter initialization can be naturally formulated as a bilevel optimization problem. 



Meanwhile, bilevel optimization has become an important tool for addressing few-shot learning problems. Beyond classical applications such as hyperparameter optimization~\cite{domke2012generic,maclaurin2015gradient,franceschi2017forward,lorraine2020optimizing}, its recent use in large-scale LLM data reweighting~\cite{pan2025scalebio} has further demonstrated its scalability and practical importance.
In its standard formulation, bilevel optimization consists of a hierarchical two-level structure, in which the outer-level problem is defined with respect to the solution of the inner-level problem,
\begin{equation}
\label{eq:bilevel_basic_form}
\begin{aligned}
\min_{\phi \in \Phi} \quad & \mathcal{L}(\phi) = L_1\bigl(\phi, w^*\bigr) \\
\text{s.t.} \quad & w^* = \arg\min_{w} \, L_2(\phi, w).
\end{aligned}
\end{equation}
A major limitation of this standard formulation is the computational burden of hypergradient estimation. 
Under the assumptions that the lower-level objective is strongly convex and that both the upper- and lower-level objectives satisfy appropriate smoothness conditions, computing or estimating the hypergradient typically involves Hessians, Jacobians, or their vector products, resulting in substantial computational and memory costs in large-scale settings.
Recent penalty methods for bilevel optimization~\cite{kwon2023fully,chen2025near,yang2026second} reformulate the original bilevel problem into an equivalent penalized optimization problem. 
This reformulation enables optimization using only first-order gradient information, while avoiding second-order information and achieving near-optimal theoretical complexity.

This naturally leads to the following question:
\textit{Is it possible to develop a nearly optimal first-order MAML-type meta-learning method?}

To answer this question, we develop ABO-Med, a first-order bilevel meta-learning framework for medical image classification. 
Specifically, ABO-Med integrates first-order penalty methods for bilevel optimization into the ANIL framework, enabling efficient MAML-type meta-learning without relying on second-order information.
Furthermore, to address the lack of medical image specific augmentation strategies in existing meta-learning methods, ABO-Med incorporates Medical Adaptive RandomAugment (MedRAug), a modality-aware augmentation strategy tailored to different types of medical imaging data.

\subsection{Contributions}
Our primary contributions are summarized as follows:
\begin{itemize}
    \item We propose ABO-Med, a scalable and flexible framework for MAML-type meta-learning in few-shot medical image classification. We establish near-optimal convergence guarantees for ABO-Med without requiring second-order derivative computations.

    \item  We empirically demonstrate the effectiveness of ABO-Med on multiple few-shot medical image classification benchmarks. ABO-Med consistently outperforms strong baselines, achieving average gains of 1.99\%, 2.06\%, and 18.76\% on Pap-Smear, ISIC2018, and BreakHis40X, respectively.

    \item We also introduce MedRAug, a medical-adaptive augmentation strategy that further improves ABO-Med by 1.69\%, 2.72\%, and 6.01\% on Pap-Smear, ISIC2018, and BreakHis40X, respectively. Additional cross-domain, ablation, and efficiency studies further verify the effectiveness and efficiency of ABO-Med.



\end{itemize}

\subsection{Related Work}

\paragraph{Bilevel Optimization.}
Bilevel optimization algorithms can generally be divided into first-order and second-order methods. Most prior second-order methods, such as Approximate Implicit Differentiation~(AID)~\cite{ghadimi2018approximation,ji2021bilevel} and Iterative Differentiation~(ITD)~\cite{lorraine2020optimizing,domke2012generic,bolte2021nonsmooth,arbel2021amortized}, estimate the hypergradient $\nabla \mathcal{L}(\phi)$ through Hessian-vector products and Jacobian-vector products. To further improve efficiency, qNBO~\cite{fang2025qnbo} adopts a quasi-Newton scheme to approximate the hypergradient. On the other hand, RAHGD~\cite{yang2023accelerating} minimizes $\mathcal{L}(\phi)$ via inexact gradient descent or (perturbed) accelerated gradient descent.
In contrast, first-order methods rely only on gradient information. PZOBO~\cite{sow2022convergence} uses a zeroth-order-like method to approximate the response Jacobian via finite differences, and thus estimates the hypergradient. By reformulating the lower-level problem as an optimality constraint, fully first-order methods have been developed. ~\cite{liu2022bome} derived a first-order hypergradient for deterministic settings, while~\cite{kwon2023fully,chen2025near} proposed $\mathrm{F}^{2}\mathrm{BA}$, which avoids Hessian and Jacobian estimation.

\paragraph{Optimization-Based Meta-Learning.}

A representative optimization-based meta-learning method is MAML~\cite{finn2017model}, which learns an initialization that can quickly adapt to new tasks. However, its meta-update requires differentiating through the inner-loop optimization process, leading to high computational and memory costs. To reduce this complexity, several MAML variants have been proposed. FOMAML~\cite{nichol2018first} simplifies the meta-gradient by removing second-order terms, while Reptile~\cite{nichol2018first} approximates meta-gradients through parameter differences between the initial and adapted models. In another direction, ANIL~\cite{raghu2019rapid} reduces computation by updating only the classifier in the inner loop, while ALFA~\cite{baik2020meta} improves fast adaptation through adaptive inner-loop optimization.
Tra-MAML~\cite{voon2025trapezoidal} improves MAML through a trapezoidal step scheduler that adjusts inner-loop adaptation steps to balance model adaptation capacity and computational efficiency.

\begin{figure*}[t]
    \centering
    \includegraphics[width=1\textwidth]{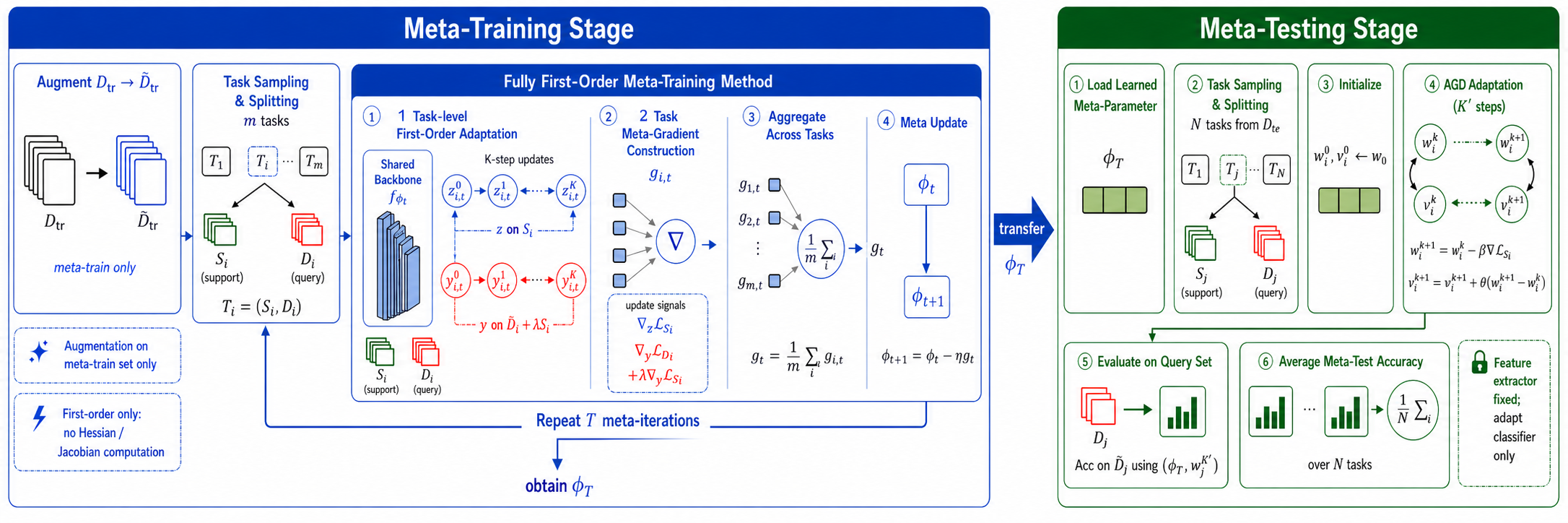}
    \caption{Overall framework of ABO-Med, including task sampling, DataAugment, fully first-order task adaptation, meta-gradient construction, and outer-loop meta-update.}
    \label{fig:framework}
\end{figure*}

\paragraph{Medical Image Augmentation.} Traditional augmentation methods~\cite{perez2018data}, such as cropping, flipping, rotation, translation, and color perturbation, are widely used to improve generalization by increasing training sample diversity. However, these methods usually require manual design and may not be equally suitable for different medical imaging modalities. AutoAugment~\cite{cubuk2019autoaugment} formulates augmentation design as a policy search problem, but its search stage is computationally expensive. RandAugment~\cite{cubuk2020randaugment} simplifies this process by removing the search stage and using only two hyperparameters: the number of operations and the augmentation magnitude. Nevertheless, augmentation policies designed for natural images may distort diagnostic structures in medical images. For pathology images, H\&E-RandomAugment~\cite{faryna2021tailoring} adapts RandAugment to pathological images by modifying its augmentation operation library and introducing H\&E stain-aware operations to better handle staining and scanner variations.

\section{Preliminaries}
We first introduce some notation used throughout the paper. For any twice differentiable function $L(\phi, w)$ with $w \in \mathbb{R}^{d_w}$ and $\phi \in \Phi$, let $\nabla_\phi L(\phi, w)$ and $\nabla_w L(\phi, w)$ denote the partial gradients with respect to $\phi$ and $w$. 
The notation $\|\cdot\|$ stands for the vector Euclidean norm. Finally, we use $\mathcal{O}(\cdot)$ to hide only absolute constants that are independent of all problem parameters, $\tilde{\mathcal{O}}(\cdot)$ to additionally suppress polylogarithmic factors, and $\Omega(\cdot)$ to denote the corresponding lower-bound order up to absolute constants.

\subsection{Problem Formulation}
\label{sec:Problem Formulation}
In few-shot classification, each task is constructed as a $k$-way, $n$-shot problem, where $k$ denotes the number of classes and $n$ denotes the number of labeled support examples per class. Thus, each task contains a small support set used for task-specific adaptation and a query set used for evaluating the adapted model.

Meta-learning proceeds over a sequence of episodes. In each episode, we sample $m$ tasks $\{\mathcal{T}_i\}_{i=1}^m$ from a task distribution $P_{\mathcal T}$, where each task $\mathcal{T}_i$ follows the above $k$-way, $n$-shot construction. For each task $\mathcal{T}_i$, let $\mathcal{S}_i$ and $\mathcal{D}_i$ denote its support and query sets, respectively. Let $\mathcal{L}(\phi,w_i;\xi)$ denote the loss on a sample $\xi$, where $\phi$ denotes the shared parameters of the embedding model and $w_i$ denotes the task-specific parameters. The goal is to learn shared parameters $\phi$ that generalize well across tasks, while each task adapts its own parameters $w_i$ based on the corresponding support set. Since the medical image datasets considered in our work are collected in advance, we focus on the finite-sum setting~\cite{antoniou2018train,raghu2019rapid}, where the task losses are defined over fixed support and query samples. The following formulation adapts Problem~\eqref{eq:bilevel_basic_form} to the few-shot meta-learning setting.
\begin{equation}
\label{eq:fewshot_learning_bilevel}
\begin{aligned}
\min_{\phi}\quad
\mathcal{L}_{\mathcal D}(\phi,\tilde w^*)
&=
\frac{1}{m}\sum_{i=1}^m
\mathcal{L}_{\mathcal D_i}(\phi,w_i^*) \\
&=
\frac{1}{m}\sum_{i=1}^m
\frac{1}{|\mathcal D_i|}
\sum_{\xi\in\mathcal D_i}
\mathcal{L}(\phi,w_i^*;\xi) \\
\text{s.t.}\quad
\tilde w^*
&=
\arg\min_{\tilde w}\,
\mathcal{L}_{\mathcal S}(\phi,\tilde w), \\
\mathcal{L}_{\mathcal S}(\phi,\tilde w)
&=
\frac{1}{m}\sum_{i=1}^m
\mathcal{L}_{\mathcal S_i}(\phi,w_i).
\end{aligned}
\end{equation}
where
$$
\mathcal{L}_{\mathcal{S}_i}(\phi,w_i)
=
\frac{1}{|\mathcal{S}_i|}
\sum_{\xi\in\mathcal{S}_i}
\mathcal{L}(\phi,w_i;\xi)
+R(w_i),
$$
with a strongly convex regularizer $R(w_i)$, e.g., an $\ell_2$ regularizer. 
Here, $\mathcal{S}_i$ and $\mathcal{D}_i$ denote the training and test datasets of task $\mathcal{T}_i$, respectively. 
The lower-level problem is equivalent to solving each $w_i^*$ as a minimizer of the task-specific support loss $\mathcal{L}_{\mathcal{S}_i}(\phi,w_i)$ for $i=1,\ldots,m$.

Here, $\tilde{w}:=(w_1,\dots,w_m)$ denotes the collection of task-specific parameters over all tasks, and $\tilde{w}^*:=(w_1^*,\dots,w_m^*)$ denotes the corresponding lower-level solution for a given shared representation parameter $\phi$. 
In our setting, $\tilde{w}$ corresponds to the parameters of the last linear layer, while $\phi$ denotes the parameters of the remaining feature extractor, such as the four convolutional layers in CNN4~\cite{snell2017prototypical} or the residual blocks in ResNet~\cite{he2016deep}. 
With the strongly convex regularizer $R(w_i)$, the lower-level objective is strongly convex with respect to $\tilde{w}$, whereas the upper-level objective $\mathcal{L}_{\mathcal{D}}(\phi,\tilde{w}^*(\phi))$ is generally nonconvex with respect to $\phi$. 
Moreover, since the support set $\mathcal{S}_i$ and query set $\mathcal{D}_i$ in few-shot learning are typically small, all task-level updates are performed using full-batch gradient descent without data resampling.

\subsection{Technical Assumptions and Definitions}
\label{sec:Technical Assumptions and Definitions}
We further assume that the loss function for each task satisfies the following standard conditions, which are typical in the literature on MAML-type optimization \cite{finn2017model,ji2020convergence,ji2022theoretical}.

\begin{assumption}
\label{assum:basic_technical_Assumpions}
The loss function $\mathcal{L}_{\mathcal S_i}(\phi,w_i)$ (inner
loop) and $\mathcal{L}_{\mathcal D_i}(\phi,w_i)$(outer loop) for each task $\mathcal T_i$ satisfy:
\begin{itemize}
    \item  $\mathcal{L}_{\mathcal S_i}(\phi,w_i)$  is $\mu$-strongly convex in $w_i$;

    \item $\mathcal{L}_{\mathcal S_i}(\phi,w_i)$ is $L_s$-gradient Lipschitz;

    \item  $\mathcal{L}_{\mathcal S_i}(\phi,w_i)$ is $\rho_s$-Hessian Lipschitz;

    \item $\mathcal{L}_{\mathcal D_i}(\phi,w_i)$ is $C_d$-Lipschitz in $w_i$;

    \item $\mathcal{L}_{\mathcal D_i}(\phi,w_i)$ is $L_d$-gradient Lipschitz;

    \item $\mathcal{L}_{\mathcal D_i}(\phi,w_i)$ is two-times continuous differentiable;

     \item $\mathcal{L}_i(\phi):=\mathcal{L}_{\mathcal D_i}(\phi,  w_i^*)$ is lower bounded, i.e.,\\ $\inf_{\phi } \mathcal{L}_i(\phi) >\! -\infty.$
\end{itemize}
\end{assumption}


\begin{definition}
\label{def:condition_number}
Under assumption~\ref{assum:basic_technical_Assumpions}, we define the largest smoothness constant $\ell := \max\{C_d, L_s, L_d, \rho_s\}$ and the condition number $\kappa := \ell / \mu$.
\end{definition}

\begin{definition}
Given a differentiable function $\mathcal{L}(\phi): \mathbb{R}^d \to \mathbb{R}$, we call $\hat{\phi}$ an $\epsilon$-first-order stationary point of $\mathcal{L}(\phi)$ if
$$
\|\nabla \mathcal{L}(\hat{\phi})\| \leq \epsilon.
$$
\end{definition}

\section{Method}

\subsection{ABO-Med}
\label{sec:ABO-Med}
\begin{algorithm}[!t]
\caption{ABO-Med}
\label{alg:ABO-Med}
\begin{algorithmic}[1]
\STATE \textbf{Input:} Meta-training dataset $D_{\mathrm{tr}}$, meta-testing dataset $D_{\mathrm{te}}$; initial parameters $\phi_0,w_0$; meta-training parameters $\eta,\alpha,\tau,\lambda,K,m,T$; meta-testing parameters $\beta,\theta,K',N$.
\STATE \textbf{Output:} The average meta-testing accuracy.
\STATE  Apply data augmentation to $D_{\mathrm{tr}}$ and obtain $\widetilde{D}_{\mathrm{tr}}$.
\STATE {\bf Meta-training Stage}
\FOR{$t=0,1,\ldots,T-1$}
    \STATE Sample and split $m$
    tasks $\{\mathcal T_i\}=(\mathcal S_i,\mathcal D_i)$ from $\widetilde{D}_{tr}$.
    \FOR{$i=1,2,\ldots,m$}
        \STATE Set $y_{i,t}^{0} \gets w_0$ and $z_{i,t}^{0} \gets w_0$.
        \FOR{$k=0,1,\ldots,K-1$}
            \STATE $z_{i,t}^{k+1} \gets z_{i,t}^{k} - \alpha \nabla_w \mathcal{L}_{\mathcal S_i}(\phi_t,z_{i,t}^{k})$.
            \STATE $y_{i,t}^{k+1} \gets y_{i,t}^{k} - \tau \bigl(\nabla_w \mathcal{L}_{\mathcal D_i}(\phi_t,y_{i,t}^{k})$
            \STATE \hspace{3.8em} $+ \lambda \nabla_w \mathcal{L}_{\mathcal S_i}(\phi_t,y_{i,t}^{k})\bigr)$.
        \ENDFOR
        \STATE $g_{i,t} \gets \nabla_{\phi}\mathcal{L}_{\mathcal D_i}(\phi_t,y_{i,t}^{K})$
        \STATE \hspace{3em} $+ \lambda \Bigl( \nabla_{\phi}\mathcal{L}_{\mathcal S_i}(\phi_t,y_{i,t}^{K}) - \nabla_{\phi}\mathcal{L}_{\mathcal S_i}(\phi_t,z_{i,t}^{K}) \Bigr)$.
    \ENDFOR
    \STATE $\phi_{t+1} \gets \phi_t - \eta \dfrac{1}{m}\sum_{i=1}^m g_{i,t}$.
\ENDFOR

\STATE {\bf Meta-testing Stage}
\STATE Load the learned meta-parameter $\phi_T$. \FOR{$i=1,2,\ldots,N$} \STATE Sample and split $N$ tasks $\mathcal T_i=(\mathcal S_i,\mathcal D_i)$ from $D_{te}$. \STATE Initialize $w_i^0 \gets w_0$ and $v_i^0 \gets w_i^0$. \FOR{$k=0,1,\ldots,K'-1$} \STATE $w_i^{k+1} \gets v_i^k - \beta \nabla_w \mathcal{L}_{\mathcal S_i}(\phi_T,v_i^k)$. \STATE $v_i^{k+1} \gets w_i^{k+1} + \theta (w_i^{k+1}-w_i^k)$. \ENDFOR \STATE Compute the accuracy on $\mathcal D_i$ using $(\phi_T,w_i^{K'})$. \ENDFOR
\STATE Report the average accuracy over the $N$ meta-test tasks.
\end{algorithmic}
\end{algorithm}

The detailed procedure of ABO-Med is presented in Algorithm~\ref{alg:ABO-Med}. 
ABO-Med consists of two stages: meta-training and meta-testing. 
The former solves a bilevel optimization problem to learn the shared meta-parameter $\phi$, while the latter evaluates the learned meta-parameter on unseen tasks. 
The meta-training parameters include the outer stepsize $\eta$, the inner learning rates $\alpha$ and $\tau$, the penalty parameter $\lambda$, the number of inner iterations $K$, the task batch size $m$, and the number of meta-training iterations $T$. 
The meta-testing parameters include the learning rate $\beta$, the momentum parameter $\theta$, the number of adaptation steps $K'$, and the number of sampled meta-testing tasks $N$.

In the meta-training stage, a mini-batch of tasks is sampled from the meta-training set at each iteration, and each task is divided into a support set and a query set. 
The meta-updates are then performed on the task-specific variables to construct approximate hypergradients, which are averaged across tasks to update the shared meta-parameter. 
In the meta-testing stage, the learned meta-parameter is fixed and transferred to unseen tasks from the meta-testing set. 
For each task, the task-specific parameter is adapted on the support set via Accelerated Gradient Descent (AGD)~\cite{nesterov2018lectures}, and the adapted model is then evaluated on the query set. 
The final performance is reported as the average accuracy over the sampled meta-testing tasks.

\subsection{Fully First-Order MAML Method}
\label{sec:F2BA}
In this section, we use the first-order penalty methods~\cite{kwon2023fully,chen2025near} to solve the bilevel optimization problem in the meta-training stage and obtain the meta-parameter $\phi_T$. We first present an equivalent penalized formulation of problem~\eqref{eq:bilevel_basic_form}.
If we let the Lagrange function be
$$
\mathcal{L}_{\lambda}(\phi,\tilde w)
:= \mathcal{L}_{\mathcal D}(\phi,\tilde w) + \lambda\bigl(\mathcal{L}_{\mathcal S}(\phi,\tilde w)-\mathcal{L}_{\mathcal S}(\phi,\tilde w^*)\bigr),
$$
where $\tilde w^*=\arg\min_{\tilde w}\,
\mathcal{L}_{\mathcal S}(\phi,\tilde w)$.~\cite{kwon2023fully} showed that Problem~\eqref{eq:bilevel_basic_form} can be effectively solved by the following formulation:
\begin{equation}
\label{eq:penalty_function}
\begin{aligned}
\min_{\phi \in \Phi} \mathcal{L}_{\lambda}^{*}(\phi)
&:= \mathcal{L}_{\lambda}\bigl(\phi,\tilde w_{\lambda}^{*}(\phi)\bigr), \\
\text{where}\quad
\tilde w_{\lambda}^{*}(\phi)
&:= \arg\min_{\tilde w} \mathcal{L}_{\lambda}(\phi,\tilde w).
\end{aligned}
\end{equation}

The penalized objective $\mathcal{L}_{\lambda}^{*}(\phi)$ provides a valid first-order approximation to the original bilevel objective $\mathcal{L}(\phi)$. 
Specifically, when $\lambda \asymp \epsilon^{-1}$, every $\epsilon$-stationary point of $\mathcal{L}_{\lambda}^{*}(\phi)$ is an $\mathcal{O}(\epsilon)$-stationary point of $\mathcal{L}(\phi)$. 
Moreover, for sufficiently large $\lambda$, the gradient Lipschitz constant of $\mathcal{L}_{\lambda}^{*}(\phi)$ becomes independent of $\lambda$. 
And the hypergradient $\nabla \mathcal{L}_{\lambda}^*(\phi)$ can be expressed as
\begin{equation}
\begin{aligned}
\nabla \mathcal{L}_{\lambda}^{*}(\phi)
&=
\nabla_\phi \mathcal{L}_{\mathcal D}\bigl(\phi,\tilde w_{\lambda}^{*}(\phi\bigr) \\
&\quad
+ \lambda\Bigl(\!
\nabla_\phi \mathcal{L}_{\mathcal S}\bigl(\phi,\tilde w_{\lambda}^{*}(\phi)\bigr)
\!-\!\nabla_\phi \mathcal{L}_{\mathcal S}\bigl(\phi,\tilde w^{*}(\phi)\bigr)
\!\Bigr).
\end{aligned}
\end{equation}
This expression only involves first-order gradients of the support and query losses, and therefore avoids the Hessian- or Jacobian-related computations required by the hypergradient of the original bilevel objective.

During the meta-training stage, auxiliary variables $\tilde y$ and $\tilde z$ are introduced to reformulate the penalty bilevel problem~\eqref{eq:penalty_function} as the following penalized optimization problem~\cite{chen2025near}:
\begin{equation}
\label{eq:auxiliary-optimization}
\begin{aligned}
\min_{\phi,\, \tilde y}\quad
& \mathcal{L}_{\mathcal D}(\phi,\tilde y)
+ \lambda \left( \mathcal{L}_{\mathcal S}(\phi,\tilde y)
- \min_{\tilde z} \mathcal{L}_{\mathcal S}(\phi,\tilde z) \right) \\
&= \min_{\phi} \mathcal{L}_\lambda^*(\phi).
\end{aligned}
\end{equation}
In the meta-training stage of ABO-Med, the bilevel optimization problem is solved via the penalized formulation~\eqref{eq:auxiliary-optimization}, by performing gradient descent jointly over $\phi$, $\tilde y$, and $\tilde z$. For a fixed $\phi$, $\tilde y \approx \tilde w_\lambda^*(\phi)$ and $\tilde z \approx \tilde w^*(\phi)$ can be obtained by applying gradient descent to $\mathcal{L}_\lambda(\phi,\cdot)$ and $\mathcal{L}_{\mathcal S}(\phi,\cdot)$, respectively.

\subsection{Theoretical Results}
\label{sec:theory}
We establish the following result to characterize both the convergence of the meta-training stage for finding $\phi_T$ and the adaptation guarantee of the meta-testing stage for obtaining task-specific parameters $w_i^{K'}$. Theorem~\ref{thm:convergence} adapts the framework of~\cite{chen2025near} to our
few-shot meta-learning setting. The proof follows the same strategy,
with parameter choices adjusted to our setting.

\begin{theorem}
\label{thm:convergence}
Suppose Assumption~\ref{assum:basic_technical_Assumpions} holds. Define
$\Delta := \mathcal{L}(\phi_0)-\inf_{\phi}\mathcal{L}(\phi)$ and
$R := \frac{1}{m}\sum_{i=1}^m
\|w_{i,0}-w_i^*(\phi_0)\|^2$.
Let $\eta \asymp \ell^{-1}\kappa^{-3}$ and
$\lambda \asymp
\max\{\kappa/\sqrt R,\ell\kappa^2/\Delta,\ell\kappa^3/\epsilon\}$.
Set the meta-training parameters in Algorithm~\ref{alg:ABO-Med} as
$\alpha=1/L_s$, $\tau=1/(2\lambda L_s)$, and
$K=\mathcal{O}(\kappa\log(\lambda \ell/\mu))$.
Then the meta-training stage finds an $\epsilon$-first-order stationary
point of $\mathcal{L}(\phi)$ within
$T=\mathcal{O}(\ell\kappa^4\epsilon^{-2}\log(\ell\kappa/\epsilon))$
first-order oracle calls, where $\ell$ and $\kappa$ are defined in
Definition~\ref{def:condition_number}. The meta-testing parameter
$K'$ is chosen separately; setting
$K'=\widetilde{\mathcal O}(\sqrt{\kappa})$ makes the AGD solver produce
a sufficiently accurate approximation of each test-task minimizer.
\end{theorem}

\section{Experiments}
In this section, we conduct a comprehensive evaluation of ABO-Med~(Algorithm~\ref{alg:ABO-Med}) from five aspects: classification performance, data augmentation ablation, backbone ablation, training efficiency, and cross-domain generalization. For classification performance, ABO-Med is compared with representative few-shot learning methods on multiple medical image datasets. The ablation studies further examine the effects of different data augmentation strategies and backbone networks, with additional backbone results provided in Appendix~\ref{app:ablation_study_backbones}. We also analyze the training efficiency of ABO-Med against representative meta-learning methods to demonstrate its advantage in reducing computational cost. In addition, cross-domain experiments are conducted to evaluate its generalization ability on unseen medical image domains under distribution shifts.

\subsection{Dataset Description}
\label{sec:data}

Experiments were conducted on three main medical image datasets: Pap-Smear~\cite{jantzen2005pap}, ISIC2018~\cite{codella2019skin}, and BreaKHis40X~\cite{spanhol2015BreakHis}. For the cross-domain experiments, BCHI~\cite{bolhasani2020histopathological} is further used as the target dataset to evaluate the generalization ability under domain shift. Pap-Smear contains 917 cervical cell images from 7 classes. ISIC2018 includes 10,015 skin lesion images from 7 classes. BreaKHis40X is a subset of the BreaKHis dataset and consists of 1,820 breast histopathological images at 40$\times$ magnification from 8 classes. BCHI contains 922 breast invasive ductal carcinoma histopathological images from 3 histological grades.The detailed dataset split is provided in the Appendix~\ref{app:dataset_split}.

\subsection{Medical Adaptive RandomAugment}

\begin{algorithm}[t]
\caption{MedRAug}
\label{alg:medical_adaptive_randaugment}
\begin{algorithmic}[1]
\STATE \textbf{Input:} Image dataset $\mathcal{D}$, modality label $p$, number of operations $N_p$, and magnitude range $\mathcal{M}^{(p)}$.
\STATE \textbf{Output:} Augmented image dataset $\widetilde{\mathcal{D}}$.
\STATE Initialize the augmented image dataset $\widetilde{\mathcal{D}}\gets \emptyset$.
\STATE Select the modality-aware operation pool $\mathcal{A}^{(p)}$, where $p\in\{{\text{BreakHis}, \text{ISIC 2018}, \text{Pap-smear}}\}$.
\FOR{each image $x \in \mathcal{D}$}
\STATE Initialize the augmented image $\tilde{x}\gets x$.
\STATE Randomly sample the magnitude $M_p \in \mathcal{M}^{(p)}$ and $N_p$ operations from the selected operation pool.
\FOR{$j=1,2,\ldots,N_p$}
\STATE Update $\tilde{x}$ by applying the $j$-th sampled operation with the parameter determined by $M_p$.
\ENDFOR
\STATE $\widetilde{\mathcal{D}} \gets \widetilde{\mathcal{D}} \cup {\tilde{x}}$.
\ENDFOR
\STATE \textbf{return} $\widetilde{\mathcal{D}}$.
\end{algorithmic}
\end{algorithm}

\begin{table*}[t]
\scriptsize
\centering
\caption{Comparison of various methods on the Pap-Smear, ISIC2018, and BreakHis40X datasets across different few-shot settings. Results are reported as mean test accuracy (\%) with 95\% confidence intervals. }
\label{tab:main_results_all_datasets}
\renewcommand{\arraystretch}{1.08}
\setlength{\tabcolsep}{2.2pt}
\begin{tabularx}{\textwidth}{>{\raggedright\arraybackslash}p{1.45cm} *{15}{>{\centering\arraybackslash}X}}
\toprule
\multirow{2}{*}{Method}
& \multicolumn{5}{c}{Pap-Smear}
& \multicolumn{5}{c}{ISIC2018}
& \multicolumn{5}{c}{BreakHis40X} \\
\cmidrule(lr){2-6} \cmidrule(lr){7-11} \cmidrule(lr){12-16}
& 1-shot & 3-shot & 5-shot & 7-shot & 10-shot
& 1-shot & 3-shot & 5-shot & 7-shot & 10-shot
& 1-shot & 3-shot & 5-shot & 7-shot & 10-shot \\
\midrule
MAML
& 64.59 $\pm$ 0.97
& 76.97 $\pm$ 0.76
& 81.78 $\pm$ 0.55
& 84.96 $\pm$ 0.45
& 87.41 $\pm$ 0.39
& 42.36 $\pm$ 0.70
& 48.09 $\pm$ 0.70
& 48.86 $\pm$ 0.65
& 51.23 $\pm$ 0.62
& 54.90 $\pm$ 0.60
& 37.88 $\pm$ 0.66
& 44.45 $\pm$ 0.65
& 47.31 $\pm$ 0.62
& 47.17 $\pm$ 0.57
& 49.31 $\pm$ 0.61 \\

FOMAML
& 64.60 $\pm$ 1.09
& 78.16 $\pm$ 0.77
& 81.21 $\pm$ 0.59
& 84.11 $\pm$ 0.48
& 85.88 $\pm$ 0.43
& 36.73 $\pm$ 0.55
& 47.64 $\pm$ 0.67
& 51.36 $\pm$ 0.65
& 52.04 $\pm$ 0.65
& 52.73 $\pm$ 0.60
& 38.15 $\pm$ 0.71
& 47.01 $\pm$ 0.69
& 49.64 $\pm$ 0.63
& 55.93 $\pm$ 0.67
& 49.63 $\pm$ 0.61 \\

Reptile
& 55.08 $\pm$ 0.88
& 68.51 $\pm$ 0.70
& 71.32 $\pm$ 0.67
& 71.44 $\pm$ 0.64
& 71.01 $\pm$ 0.64
& 40.77 $\pm$ 0.68
& 46.85 $\pm$ 0.68
& 48.93 $\pm$ 0.63
& 49.02 $\pm$ 0.60
& 47.88 $\pm$ 0.57
& 40.86 $\pm$ 0.74
& 40.86 $\pm$ 0.74
& 46.59 $\pm$ 0.68
& 46.47 $\pm$ 0.65
& 46.50 $\pm$ 0.68 \\

ALFA
& 66.40 $\pm$ 1.16
& 78.33 $\pm$ 0.59
& 82.17 $\pm$ 0.46
& 81.85 $\pm$ 0.42
& 87.56 $\pm$ 0.36
& 42.73 $\pm$ 0.75
& 52.57 $\pm$ 0.64
& 57.12 $\pm$ 0.63
& 56.14 $\pm$ 0.64
& 61.57 $\pm$ 0.63
& 37.09 $\pm$ 0.64
& 52.41 $\pm$ 0.66
& 54.28 $\pm$ 0.64
& 56.16 $\pm$ 0.64
& 60.38 $\pm$ 0.57 \\

Tra-MAML
& 66.40 $\pm$ 1.02
& 81.48 $\pm$ 0.64
& 84.42 $\pm$ 0.47
& 86.71 $\pm$ 0.40
& 87.51 $\pm$ 0.37
& 41.71 $\pm$ 0.66
& 48.67 $\pm$ 0.68
& 55.20 $\pm$ 0.66
& 61.25 $\pm$ 0.61
& 62.84 $\pm$ 0.56
& 38.45 $\pm$ 0.67
& 44.96 $\pm$ 0.62
& 57.76 $\pm$ 0.66
& 62.61 $\pm$ 0.67
& 62.67 $\pm$ 0.60 \\

ABO-Med
& \underline{70.78 $\pm$} \underline{0.72}
& \underline{81.51 $\pm$} \underline{0.68}
& \underline{86.78 $\pm$} \underline{0.65}
& \underline{87.10 $\pm$} \underline{0.64}
& \underline{87.62 $\pm$} \underline{0.61}
& \underline{44.17 $\pm$} \underline{0.74}
& \underline{53.56 $\pm$} \underline{0.65}
& \underline{58.28 $\pm$} \underline{0.62}
& \underline{62.32 $\pm$} \underline{0.66}
& \underline{63.67 $\pm$} \underline{0.64}
& \underline{53.44 $\pm$} \underline{1.14}
& \underline{55.37 $\pm$} \underline{0.83}
& \underline{69.07 $\pm$} \underline{0.79}
& \underline{72.77 $\pm$} \underline{0.76}
& \underline{76.17 $\pm$} \underline{0.69} \\

\shortstack[l]{ABO-Med\\(MedRAug)}
& \textbf{75.46 $\pm$ 0.71}
& \textbf{83.11 $\pm$ 0.67}
& \textbf{87.42 $\pm$ 0.65}
& \textbf{88.37 $\pm$ 0.62}
& \textbf{90.45 $\pm$ 0.58}
& \textbf{52.76 $\pm$ 0.96}
& \textbf{55.24 $\pm$ 0.72}
& \textbf{61.42 $\pm$ 0.65}
& \textbf{63.86 $\pm$ 0.64}
& \textbf{67.02 $\pm$ 0.58}
& \textbf{62.04 $\pm$ 0.97}
& \textbf{65.81 $\pm$ 0.89}
& \textbf{72.61 $\pm$ 0.84}
& \textbf{77.83 $\pm$ 0.85}
& \textbf{80.21 $\pm$ 0.78} \\

\bottomrule
\end{tabularx}
\end{table*}

Data augmentation strategies designed for natural images may not be directly suitable for medical images, where classification often relies on modality-specific structures such as tissue morphology, lesion boundaries, and cellular morphology. To address this issue, we propose Medical Adaptive RandomAugment~(MedRAug), which retains the randomized mechanism of RandomAugment~\cite{cubuk2020randaugment} while adapting augmentation operations and magnitudes to different medical imaging modalities. This design increases sample diversity while better preserving critical structural information.

MedRAug extends RandomAugment to medical imaging by using modality-specific augmentation operations and magnitude ranges. For each image, it selects the corresponding augmentation library according to the modality, randomly samples a fixed number of operations, and applies them sequentially. The magnitude is constrained within a modality-specific range to avoid unrealistic transformations and preserve clinically relevant structures. The pseudocode of MedRAug is shown in Algorithm~\ref{alg:medical_adaptive_randaugment}.

Under this framework, we design modality-specific augmentation strategies for three types of medical images. For Pap-smear cytology images, we apply mild microscopic perturbations while preserving cellular morphology~\cite{austin2025cnn}. For ISIC2018 skin lesion images, we simulate variations in lesion morphology, imaging angle, and illumination~\cite{gessert2020skin,perez2018data}. For BreakHis40X pathological images, we preserve tissue structures while modeling staining and scanning variations~\cite{faryna2021tailoring}. The detailed augmentation operations and design rationales are provided in Appendix~\ref{app:augmentation_space}.

\paragraph{Ablation Study on Data Augmentation}
\label{sec:ablation}
An ablation study under different data augmentation strategies is conducted, as shown in Appendix~\ref{app:aug_ablation_results}. Specifically, we compare the baseline augmentation setting~\cite{voon2025trapezoidal} with MedRAug, CutMix~\cite{yun2019cutmix}, Mixup~\cite{zhang2017mixup}, and Cutout~\cite{devries2017cutout}. The principles of these augmentation methods are described in Appendix~\ref{app:data_augment_strategy}.

\begin{table}[t]
\centering
\caption{Comparison of training time (h) per 100 tasks between few-shot learning methods and ABO-Med on the Pap-Smear, ISIC2018, and BreakHis40X datasets.}
\label{tab:time_comparison}
\renewcommand{\arraystretch}{0.95}
\setlength{\tabcolsep}{2pt}
\scriptsize
\begin{tabularx}{\columnwidth}{
@{}
>{\raggedright\arraybackslash}p{1.42cm}
>{\raggedright\arraybackslash}p{1.38cm}
*{6}{>{\centering\arraybackslash}X}
@{}
}
\toprule
Dataset & Method & \multicolumn{5}{c}{Shot} & Avg \\
\cmidrule(lr){3-7}
& & 1 & 3 & 5 & 7 & 10 & \\
\midrule

\multirow{6}{*}{Pap-Smear}
& MAML
& 0.0711 & 0.0745 & 0.0760 & 0.0773 & 0.0801 & 0.0758 \\

& FOMAML
& 0.0674 & 0.0706 & 0.0726 & 0.0732 & 0.0738 & 0.0715 \\

& Reptile
& 0.0098 & 0.0112 & 0.0112 & 0.0134 & 0.0147 & 0.0121 \\

& ALFA
& 0.0716 & 0.0820 & 0.0865 & 0.0956 & 0.1010 & 0.0873 \\

& Tra-MAML
& 0.0597 & 0.0644 & 0.0696 & 0.0717 & 0.0722 & 0.0675 \\

& ABO-Med
& \textbf{0.0048} & \textbf{0.0104} & \textbf{0.0109} & \textbf{0.0111} & \textbf{0.0116} & \textbf{0.0098} \\
\midrule

\multirow{6}{*}{ISIC2018}
& MAML
& 0.2068 & 0.2313 & 0.2366 & 0.2371 & 0.2373 & 0.2298 \\

& FOMAML
& 0.2041 & 0.2283 & 0.2302 & 0.2313 & 0.2318 & 0.2251 \\

& Reptile
& 0.0683 & 0.0712 & 0.0783 & 0.1172 & 0.1205 & 0.0911 \\

& ALFA
& 0.1976 & 0.2176 & 0.2271 & 0.2339 & 0.2425 & 0.2237 \\

& Tra-MAML
& 0.2158 & 0.2303 & 0.2325 & 0.2332 & 0.2358 & 0.2295 \\

& ABO-Med
& \textbf{0.0210} & \textbf{0.0231} & \textbf{0.0242} & \textbf{0.0269} & \textbf{0.0273} & \textbf{0.0245} \\
\midrule

\multirow{6}{*}{BreakHis40X}
& MAML
& 0.1362 & 0.1561 & 0.1564 & 0.1589 & 0.1605 & 0.1536 \\

& FOMAML
& 0.1346 & 0.1543 & 0.1553 & 0.1562 & 0.1578 & 0.1516 \\

& Reptile
& 0.0471 & 0.0606 & 0.0968 & 0.1164 & 0.1187 & 0.0879 \\

& ALFA
& 0.1214 & 0.1564 & 0.1683 & 0.1685 & 0.1698 & 0.1569 \\

& Tra-MAML
& 0.1393 & 0.1517 & 0.1552 & 0.1547 & 0.1559 & 0.1514 \\

& ABO-Med
& \textbf{0.0186} & \textbf{0.0208} & \textbf{0.0229} & \textbf{0.0276} & \textbf{0.0295} & \textbf{0.0239} \\
\bottomrule
\end{tabularx}
\end{table}

\subsection{Result Analysis for Classification Performance}
\label{sec:result_analysis}


\begin{table*}[t]
\footnotesize
\centering
\caption{Cross-domain few-shot classification results on the BCHI dataset under near-domain and far-domain settings. Results are reported as mean test accuracy (\%) with 95\% confidence intervals.}
\label{tab:cross_domain_breakhis40x_to_bchi}
\renewcommand{\arraystretch}{1.08}
\setlength{\tabcolsep}{5.5pt}
\begin{tabularx}{\textwidth}{
>{\raggedright\arraybackslash}p{1.6cm}
*{6}{>{\centering\arraybackslash}X}
}
\toprule
\multirow{2}{*}{Method}
& \multicolumn{3}{c}{BreakHis40X $\rightarrow$ BCHI (Near-Domain)}
& \multicolumn{3}{c}{Pap-Smear $\rightarrow$ BCHI (Far-Domain)} \\
\cmidrule(lr){2-4} \cmidrule(lr){5-7}
& 3-shot & 5-shot & 10-shot
& 3-shot & 5-shot & 10-shot \\
\midrule

MAML
& 53.45 $\pm$ 0.66
& 56.13 $\pm$ 0.63
& 58.27 $\pm$ 0.62
& 47.46 $\pm$ 0.75
& 48.53 $\pm$ 0.70
& 56.18 $\pm$ 0.62 \\

FOMAML
& 57.80 $\pm$ 0.76
& 62.24 $\pm$ 0.62
& 62.89 $\pm$ 0.52
& 46.08 $\pm$ 0.78
& 50.14 $\pm$ 0.72
& 56.31 $\pm$ 0.61 \\

Reptile
& 55.27 $\pm$ 0.70
& 55.69 $\pm$ 0.64
& 56.78 $\pm$ 0.62
& 54.47 $\pm$ 0.69
& 55.08 $\pm$ 0.66
& 56.52 $\pm$ 0.58 \\

ALFA
& 63.67 $\pm$ 0.67
& 65.13 $\pm$ 0.58
& 68.67 $\pm$ 0.53
& 46.03 $\pm$ 0.77
& 45.12 $\pm$ 0.73
& 53.06 $\pm$ 0.64 \\

Tra-MAML
& 54.67 $\pm$ 0.68
& 68.68 $\pm$ 0.56
& 68.62 $\pm$ 0.52
& 52.48 $\pm$ 0.70
& 56.07 $\pm$ 0.62
& 57.46 $\pm$ 0.56 \\

ABO-Med
& \textbf{65.21 $\pm$ 0.74}
& \textbf{68.93 $\pm$ 0.69}
& \textbf{72.14 $\pm$ 0.63}
& \textbf{58.22 $\pm$ 0.71}
& \textbf{62.21 $\pm$ 0.63}
& \textbf{65.45 $\pm$ 0.56} \\

\bottomrule
\end{tabularx}
\end{table*}

ABO-Med was compared with representative few-shot learning methods, including FOMAML~\cite{finn2017model}, MAML~\cite{finn2017model}, Reptile~\cite{nichol2018first}, ALFA~\cite{baik2020meta}, and the state-of-the-art Tra-MAML~\cite{voon2025trapezoidal}, on three medical imaging datasets. For fair comparison, all methods followed the same experimental protocol, with details provided in Appendix~\ref{app:experiment_details}. As shown in Table~\ref{tab:main_results_all_datasets}, ABO-Med achieves stronger overall performance, improving over the best competing method by 1.99\%, 2.06\%, and 18.76\% on Pap-Smear, ISIC2018, and BreakHis40X, respectively. With MedRAug, ABO-Med further improves the average accuracy by 2.20\%, 3.66\%, and 6.34\%, achieving the best results across all datasets and shot settings. The larger gain on BreakHis40X may be attributed to its clearer category-level differences in tissue structure, cellular density, and texture patterns.

These results indicate that ABO-Med achieves consistently strong performance for few-shot medical image classification across datasets and shot settings. Its advantage mainly comes from the first-order penalty bilevel optimization method, which avoids the direct truncation of second-order terms in FOMAML and the costly second-order hypergradient approximation in MAML-type methods. By relying only on first-order information, ABO-Med reduces the estimation associated with second-order information, while mitigating convergence instability caused by truncated or inaccurate second-order approximations.

\subsection{Training Time Analysis}
\label{sec:time_analysis}
To further compare the computational efficiency of different methods, we measure the training time per 100 tasks under the same 3-way 1-, 3-, 5-, 7-, and 10-shot settings. To ensure a fair comparison, all methods use the same CNN4 backbone and the baseline augmentation setting. The results are reported in Table~\ref{tab:time_comparison}. Compared with the best competing method, ABO-Med reduces the average training time by approximately 19.0\%, 73.1\%, and 72.8\% on the Pap-Smear, ISIC2018, and BreakHis40X datasets, respectively. These results demonstrate that ABO-Med can substantially reduce training cost while maintaining competitive classification performance.

These results further demonstrate the computational advantage of ABO-Med. During meta-training, conventional MAML-type methods usually require computing or approximating hypergradients that involve second-order information, leading to substantial computation, estimation, and memory costs. In contrast, ABO-Med employs a first-order penalty method for bilevel optimization, where the resulting hypergradient only relies on first-order information and thus avoids expensive second-order derivative computation and storage. In addition, ABO-Med follows the ANIL framework by updating only the task-specific classifier while keeping the shared feature extractor fixed during task adaptation, which further reduces the number of updated parameters and simplifies the inner-loop optimization.

\subsection{Cross-domain Analysis}
\label{sec:cross_domain}
In real-world medical applications, models often face domain shifts arising from variations in data sources, clinical sites, and image acquisition protocols. To assess the robustness of ABO-Med under domain shifts, we further conduct cross-domain experiments to evaluate its  generalization ability. Specifically, we conduct 3-way 3-, 5-, and 10-shot cross-domain experiments by meta-training the methods on BreakHis40X or Pap-Smear and meta-testing them on BCHI, with BreakHis40X $\rightarrow$ BCHI as the near-domain scenario and Pap-Smear $\rightarrow$ BCHI as the far-domain scenario. To ensure a fair comparison, all methods use the same CNN4 backbone and the baseline augmentation setting. The results are reported in Table~\ref{tab:cross_domain_breakhis40x_to_bchi}. ABO-Med achieves the best accuracy in both settings, improving over the strongest baseline by 2.16\% in the near-domain setting and 10.64\% in the far-domain setting. This indicates that ABO-Med has stronger adaptability under both mild and severe domain shifts.

Cross-domain results suggest that ABO-Med learns more stable and transferable representations, rather than overfitting to source-domain-specific imaging styles. These representations may help preserve task-relevant discriminative features when the model is transferred to unseen target domains. In the near-domain setting, such representations can exploit shared morphological information between breast histopathology datasets. In the far-domain setting, where the source and target domains differ substantially, direct transfer of tissue-specific features becomes difficult. ABO-Med may therefore reduce its reliance on source-domain-specific patterns, supporting more reliable adaptation when only limited target samples are available.

\section{Conclusion}
In this paper, we proposed ABO-Med, a scalable and flexible MAML-type meta-learning framework for few-shot medical image classification. We further established convergence guarantees and complexity results for the proposed algorithm, demonstrating its provable efficiency under a fully first-order bilevel formulation. Experiments on several public medical datasets showed that ABO-Med consistently outperforms prior baselines, with improvements of around 1.99\% to 18.76\% over previous state-of-the-art methods. In addition, MedRAug further improves the average accuracy by around 2.20\% to 6.34\%, with its effectiveness validated through augmentation ablation studies. Overall, ABO-Med provides an efficient and effective solution for few-shot medical image classification, where stable adaptation and computational scalability are both critical.

\section*{Acknowledgment}
Yating Liu is supported by the Gansu Provincial Joint Scientific Research Fund (26JRRA1029).

\bibliography{aaai2027}

\clearpage
\appendix
\setcounter{secnumdepth}{2}
\section{Appendix}

\paragraph{Limitations and Future Work}
A limitation of ABO-Med is its relatively large number of hyperparameters,
which require careful tuning across tasks and datasets. This increases
implementation complexity and may limit its practicality in real-world
few-shot classification settings where extensive validation is costly.
In future work, we aim to develop adaptive bilevel optimization methods
~\cite{huang2021biadam,fan2023bisls,antonakopoulos2025adaptive,yang2024tuning,zhai2025problem}
and explore agent-based automated workflows for selecting and adjusting
key hyperparameters~\cite{rao2025two,ma2026intragent,rao2026fragfuse}.

\subsection{Extend Related Work}
This section provides additional discussion on optimization-based meta-learning methods that complement the related work in the main text. 
Meta-SGD~\cite{li2017meta} extends MAML by learning parameter-wise learning rates for fast adaptation, while MAML++~\cite{antoniou2018train} improves its training stability through a set of optimization refinements. 
Implicit-gradient methods such as iMAML~\cite{rajeswaran2019meta} use implicit differentiation to compute meta-gradients from the inner-level solution, thereby avoiding backpropagation through the inner optimization path and reducing memory overhead.
FO-B-MAML~\cite{chayti2024new} introduces a fully first-order bilevel approach that regularizes the inner problem and approximates hypergradients via finite differences.
Since these methods do not outperform the representative baselines and state-of-the-art methods considered in our experiments, we do not include them in the main empirical comparison.

\subsection{Proof for Theoretical Results}
\begin{proposition}(Lemma~3.1 in~\cite{kwon2023fully}; Lemma~4.1 in~\cite{chen2025near})
\label{pro:5}
Suppose Assumption~\ref{assum:basic_technical_Assumpions} holds. Define $\ell, \kappa$ according to Definition~\ref{def:condition_number}, and $\mathcal{L}_{\lambda,i}^{*}(\phi)$ according to Equation~\eqref{eq:penalty_function}. Set $\lambda \geq 2L_d/\mu$, then it holds that
\begin{enumerate}[label=(\alph*)]
    \item $\|\nabla \mathcal{L}_{\lambda,i}^{*}(\phi) - \nabla \mathcal{L}_{i}(\phi)\| = \mathcal{O}(\ell \kappa^3 / \lambda)$, \quad $\forall \phi \in \Phi$ 
    
    \item $| \mathcal{L}_{\lambda,i}^{*}(\phi)- \mathcal{L}_{i}(\phi)| = \mathcal{O}(\ell \kappa^2 / \lambda)$, \quad $\forall \phi \in \Phi$ .
    
    \item $ \mathcal{L}_{\lambda,i}^{*}(\phi)$ is $\mathcal{O}(\ell \kappa^3)$-gradient Lipschitz .
\end{enumerate}
\end{proposition}

\begin{lemma}[Lemma 2,~\cite{wang2020improved}]
Running the AGD procedure in the meta-testing stage of Algorithm~\ref{alg:ABO-Med} on an $\ell$-smooth and $\mu$-strongly-convex objective function $\mathcal{L}_{\mathcal S_i}(\phi,\cdot)$ with parameters
$\beta = \frac{1}{\ell}$ and $
\theta = \frac{\sqrt{\kappa}-1}{\sqrt{\kappa}+1}$
produces the output $w_i^{K'}$ satisfying
$$\|w_i^{K'} - w_i^*\|^2
\le
(\kappa + 1)
\left(1 - \frac{1}{\sqrt{\kappa}}\right)^{K'}
\|w_i^{0} - w_i^*\|^2,$$
where
$w_i^* = \arg\min_{w_i}\mathcal{L}_{\mathcal S_i}(\phi, w_i)$ and $\kappa = \frac{\ell}{\mu}$. We set $K'=\tilde O(\sqrt{\kappa})$ so that the AGD solver attains a sufficiently accurate approximation of $w_i^*$.
\end{lemma}

For each task $\mathcal T_i$ and a fixed representation parameter
$\phi$, write
$$
    w_i^*(\phi)
    :=
    \argmin_{w_i}\mathcal L_{\mathcal S_i}(\phi,w_i),
    \mathcal L_i(\phi)
    :=
    \mathcal L_{\mathcal D_i}(\phi,w_i^*(\phi)).
$$
Following the penalty/value-function reformulation, define
$$
    \mathcal L_{\lambda,i}(\phi,w_i)
    :=
    \mathcal L_{\mathcal D_i}(\phi,w_i)
    +\lambda\Bigl(
        \mathcal L_{\mathcal S_i}(\phi,w_i)
        -\mathcal L_{\mathcal S_i}(\phi,w_i^*(\phi))
    \Bigr),
$$
$$
    \mathcal L_{\lambda,i}^*(\phi)
    :=
    \min_{w_i}\mathcal L_{\lambda,i}(\phi,w_i),
    w_{\lambda,i}^*(\phi)
    :=
    \argmin_{w_i}\mathcal L_{\lambda,i}(\phi,w_i).
$$
The averaged true and proxy objectives are
$$
    \mathcal L(\phi)
    :=
    \frac1m\sum_{i=1}^m \mathcal L_i(\phi),
    \qquad
    \mathcal L_\lambda^*(\phi)
    :=
    \frac1m\sum_{i=1}^m
    \mathcal L_{\lambda,i}^*(\phi).
$$
For each task, the proxy gradient is
$$
    \begin{aligned}
    \nabla \mathcal L_{\lambda,i}^*(\phi)
    &=
    \nabla_\phi \mathcal L_{\mathcal D_i}
        (\phi,w_{\lambda,i}^*(\phi))\\
    &\quad
    +\lambda\Bigl(
        \nabla_\phi \mathcal L_{\mathcal S_i}
            (\phi,w_{\lambda,i}^*(\phi))
        -\nabla_\phi \mathcal L_{\mathcal S_i}(\phi,w_i^*(\phi))
    \Bigr).
    \end{aligned}
$$
\begin{lemma}[Penalty minimizer closeness]
\label{lem:penalty-minimizer-closeness}
Suppose Assumption 1 holds.  For any task $i$ and fixed $\phi\in\Phi$,
if
$\lambda\ge 2L_d/\mu$, then
$$
    \left\lVert w_{\lambda,i}^*(\phi)-w_i^*(\phi)\right\rVert
    \le
    \frac{C_d}{\lambda\mu},
$$
where $C_d$ is the Lipschitz constant of
$\mathcal L_{\mathcal D_i}$ with respect to $w_i$.  Consequently,
$$
    \left\lVert w_{\lambda,i}^*(\phi)-w_i^*(\phi)\right\rVert^2
    =
    O\!\left(\frac{\kappa^2}{\lambda^2}\right).
$$
\end{lemma}

\begin{proof}
By the first-order optimality condition for
$w_{\lambda,i}^*(\phi)$, we have
$$
    \nabla_{w_i}\mathcal L_{\mathcal D_i}
        (\phi,w_{\lambda,i}^*(\phi))
    +\lambda\nabla_{w_i}\mathcal L_{\mathcal S_i}
        (\phi,w_{\lambda,i}^*(\phi))
    =0.
$$
Since
$w_i^*(\phi)=
\argmin_{w_i}\mathcal L_{\mathcal S_i}(\phi,w_i)$
and
$\mathcal L_{\mathcal S_i}(\phi,\cdot)$ is $\mu$-strongly convex,
$$
    \left\lVert w_{\lambda,i}^*(\phi)-w_i^*(\phi)\right\rVert
    \le
    \frac1\mu
    \left\lVert
        \nabla_{w_i}
        \mathcal L_{\mathcal S_i}
            (\phi,w_{\lambda,i}^*(\phi))
    \right\rVert.
$$
Using the optimality condition above,
$$
    \left\lVert
        \nabla_{w_i}
        \mathcal L_{\mathcal S_i}
            (\phi,w_{\lambda,i}^*(\phi))
    \right\rVert
    =
    \frac1\lambda
    \left\lVert
        \nabla_{w_i}
        \mathcal L_{\mathcal D_i}
            (\phi,w_{\lambda,i}^*(\phi))
    \right\rVert
    \le
    \frac{C_d}{\lambda}.
$$
Thus
$$
    \left\lVert w_{\lambda,i}^*(\phi)-w_i^*(\phi)\right\rVert
    \le
    \frac{C_d}{\lambda\mu}.
$$
Because $C_d\le \ell$ and $\kappa=\ell/\mu$, the squared bound is
$O(\kappa^2/\lambda^2)$.
\end{proof}

\begin{lemma}[Inexact descent]
\label{lem:inexact-descent}
Let $F$ be $L$-gradient Lipschitz, and let
$x^+=x-\eta\widehat g$ with $\eta\le 1/(2L)$.  Then
\begin{align*}
    F(x^+)
    &\le
    F(x)
    -\frac{\eta}{2}\left\lVert \nabla F(x)\right\rVert^2
    -\frac{1}{4\eta}\left\lVert x^+-x\right\rVert^2\\
    &\quad
    +\frac{\eta}{2}
    \left\lVert \widehat g-\nabla F(x)\right\rVert^2 .
\end{align*}
\end{lemma}

\begin{proof}
By the smoothness of $F$,
$$
    F(x^+)
    \le
    F(x)
    +\langle \nabla F(x),x^+-x\rangle
    +\frac{L}{2}\left\lVert x^+-x\right\rVert^2 .
$$
Since $x^+-x=-\eta\widehat g$, we have
$$
    \widehat g=-\frac{x^+-x}{\eta}.
$$
Writing $e=\widehat g-\nabla F(x)$, the identity
$$
    2\langle a,b\rangle
    =
    \left\lVert a\right\rVert^2
    +\left\lVert b\right\rVert^2
    -\left\lVert a-b\right\rVert^2
$$
gives
$$
    \langle\nabla F(x),x^+\!-x\rangle
    =\!
    -\frac{\eta}{2}\!\left\lVert \nabla F(x)\right\rVert^2
    -\frac{1}{2\eta}\!\left\lVert x^+\!-\!x\right\rVert^2
    +\frac{\eta}{2}\!\left\lVert e\right\rVert^2 .
$$
Combining the two displays and using
$L/2\le 1/(4\eta)$ yields the claim.
\end{proof}

The following statement is Theorem 4 in the main text; we restate it
here for completeness before giving the proof.

\begin{theorem}
Suppose Assumption~\ref{assum:basic_technical_Assumpions} holds. Define
$\Delta:=\mathcal L(\phi_0)-\inf_\phi\mathcal L(\phi)$ and
$R:=\frac1m\sum_{i=1}^m
\left\lVert w_{i,0}-w_i^*(\phi_0)\right\rVert^2$.
Let $\eta\asymp\ell^{-1}\kappa^{-3}$ and
$\lambda\asymp
\max\{\kappa/\sqrt R,\ell\kappa^2/\Delta,\ell\kappa^3/\epsilon\}$.
Set the meta-training parameters
$\alpha=1/L_s$, $\tau=1/(2\lambda L_s)$, and
$K=O(\kappa\log(\lambda\ell/\mu))$.
Then the meta-training stage of Algorithm 1 finds an
$\epsilon$-first-order stationary point of $\mathcal L(\phi)$ within
$O(\ell\kappa^4\epsilon^{-2}\log(\ell\kappa/\epsilon))$
first-order oracle calls.
\end{theorem}

The meta-testing parameter $K'$ is not part of the above
meta-training oracle complexity.  It is chosen separately in the
meta-testing stage; by Lemma 6, setting
$K'=\widetilde O(\sqrt\kappa)$ makes the AGD solver produce a
sufficiently accurate approximation of each test-task minimizer
$w_i^*(\phi_T)$.

\begin{proof}
Let $L_\lambda$ be the gradient Lipschitz constant of the averaged
proxy $\mathcal L_\lambda^*(\phi)$.  Proposition 5 applies to each
task objective $\mathcal L_{\lambda,i}^*(\phi)$.  Averaging the
task-wise bounds and using the choice of $\lambda$ (enlarging the
hidden constant so that $\lambda\ge 2L_d/\mu$), we have
\begin{enumerate}[label=\alph*.]
    \item
    $\sup_{\phi\in\Phi}
    \left\lVert
        \nabla\mathcal L_\lambda^*(\phi)-\nabla\mathcal L(\phi)
    \right\rVert=O(\epsilon)$.
    \item
    $\mathcal L_\lambda^*(\phi_0)
    -\inf_{\phi}\mathcal L_\lambda^*(\phi)
    \le
    \Delta
    +2\sup_{\phi}
    \left|\mathcal L_\lambda^*(\phi)-\mathcal L(\phi)\right|
    =O(\Delta)$.
    \item
    $L_\lambda
    :=
    \sup_{\phi\in\Phi}
    \left\lVert \nabla^2\mathcal L_\lambda^*(\phi)\right\rVert
    =O(\ell\kappa^3)$.
\end{enumerate}
We next control the initial distance to the two inner minimizers
task by task.  By Lemma~\ref{lem:penalty-minimizer-closeness},
$$
    \frac1m\sum_{i=1}^m
    \left\lVert w_{\lambda,i}^*(\phi_0)-w_i^*(\phi_0)\right\rVert^2
    =
    O\!\left(\frac{\kappa^2}{\lambda^2}\right).
$$
Since $\lambda\gtrsim\kappa/\sqrt R$, this bound is $O(R)$.  Therefore,
by the triangle inequality and
$(a+b)^2\le 2a^2+2b^2$,
$$
    \begin{aligned}
    &\frac1m\sum_{i=1}^m
    \left\lVert w_{i,0}-w_{\lambda,i}^*(\phi_0)\right\rVert^2\\
    &\quad \le
    \frac2m\sum_{i=1}^m
    \left\lVert w_{i,0}-w_i^*(\phi_0)\right\rVert^2\\
    &\qquad
    +\frac2m\sum_{i=1}^m
    \left\lVert w_i^*(\phi_0)-w_{\lambda,i}^*(\phi_0)\right\rVert^2
    =
    O(R).
    \end{aligned}
$$
Together with the definition of $R$, this yields
$$
    \begin{aligned}
    \frac1m\sum_{i=1}^m
    \left\lVert w_{i,0}-w_{\lambda,i}^*(\phi_0)\right\rVert^2
    &+
    \frac1m\sum_{i=1}^m
    \left\lVert w_{i,0}-w_i^*(\phi_0)\right\rVert^2\\
    =O(R).
    \end{aligned}
$$
Thus it remains to show that Algorithm 1 finds an
$O(\epsilon)$-stationary point of the proxy objective
$\mathcal L_\lambda^*$.

At outer iteration $t$, for each sampled meta-training task
$\mathcal T_i$, define
$$
    \begin{aligned}
    w_i^*(\phi)
    &:=
    \argmin_w\mathcal L_{\mathcal S_i}(\phi,w),\\
    w_{\lambda,i}^*(\phi)
    &:=
    \argmin_w
    \left\{
        \mathcal L_{\mathcal D_i}(\phi,w)
        +\lambda \mathcal L_{\mathcal S_i}(\phi,w)
    \right\}.
    \end{aligned}
$$
The iterates $y_{i,t}^K$ and $z_{i,t}^K$ in Algorithm 1 are the
$K$-step outputs for approximating $w_{\lambda,i}^*(\phi_t)$ and
$w_i^*(\phi_t)$, respectively.  The first-order estimator used for the
update of $\phi$ is exactly
$$
    \begin{aligned}
    \widehat G_t
    &:=
    \frac1m\sum_{i=1}^m g_{i,t},\\
    g_{i,t}
    &:=
    \nabla_\phi\mathcal L_{\mathcal D_i}(\phi_t,y_{i,t}^K)
    +\lambda\Bigl(
        \nabla_\phi\mathcal L_{\mathcal S_i}
            (\phi_t,y_{i,t}^K)\\
    &\qquad\qquad
        -\nabla_\phi\mathcal L_{\mathcal S_i}
            (\phi_t,z_{i,t}^K)
    \Bigr).
    \end{aligned}
$$
and the outer update is
$\phi_{t+1}=\phi_t-\eta\widehat G_t$.
Applying Lemma~\ref{lem:inexact-descent} to
$F=\mathcal L_\lambda^*$, $x=\phi_t$, $x^+=\phi_{t+1}$, and
$\widehat g=\widehat G_t$ gives
\begin{align}
    \mathcal L_\lambda^*(\phi_{t+1})
    &\le
    \mathcal L_\lambda^*(\phi_t)
    -\frac{\eta}{2}
    \left\lVert\nabla\mathcal L_\lambda^*(\phi_t)\right\rVert^2        \notag\\
    &\quad
    -\frac{1}{4\eta}
    \left\lVert\phi_{t+1}-\phi_t\right\rVert^2
    +\frac{\eta}{2}
    \left\lVert
        \widehat G_t-\nabla\mathcal L_\lambda^*(\phi_t)
    \right\rVert^2 .
    \label{eq:descent}
\end{align}
For $k\in\{0,K\}$, set
$$E^y_{i,t,k}:=
    \left\lVert y_{i,t}^k-w_{\lambda,i}^*(\phi_t)\right\rVert,
    E^z_{i,t,k}:=
    \left\lVert z_{i,t}^k-w_i^*(\phi_t)\right\rVert .
$$
Also set
$$
    \mathcal E_t
    :=
    \left\lVert
        \widehat G_t-\nabla\mathcal L_\lambda^*(\phi_t)
    \right\rVert .
$$
Using the Lipschitz continuity of the partial gradients in $w$, we
first obtain
\begin{align*}
    \mathcal E_t
    &\le
    \frac1m\sum_{i=1}^m
    \left(
        (L_d+\lambda L_s)E^y_{i,t,K}
        +\lambda L_sE^z_{i,t,K}
    \right).
\end{align*}
Since $L_d,L_s\le\ell$ and $\lambda\ge1$ after enlarging the absolute
constant in the choice of $\lambda$ if necessary, the right-hand side
is bounded by
\begin{align*}
    \mathcal E_t
    &\le
    \frac1m\sum_{i=1}^m
    \left(
        2\lambda\ell E^y_{i,t,K}
        +\lambda\ell E^z_{i,t,K}
    \right).
\end{align*}
By Jensen's inequality and $(2a+b)^2\le 8a^2+2b^2$,
\begin{align*}
    \mathcal E_t^2
    &\le
    c_0\lambda^2\ell^2
    \frac1m\sum_{i=1}^m
    \left((E^y_{i,t,K})^2+(E^z_{i,t,K})^2\right)
\end{align*}
for a universal constant $c_0>0$.
The $K$ inner gradient steps satisfy
$$
    (E^y_{i,t,K})^2
    \le
    q(E^y_{i,t,0})^2,
    \qquad
    (E^z_{i,t,K})^2
    \le
    q(E^z_{i,t,0})^2,
$$
where
$$
    q:=\exp\!\left(-\frac{\mu K}{4\ell}\right)
$$
is a common upper bound on the contraction factors obtained from the
linear convergence of gradient descent on the penalized lower problem
and on the support-loss subproblem.
Define the tracking error
$$
    \delta_t:=
    \frac1m\sum_{i=1}^m
    \left((E^y_{i,t,0})^2+(E^z_{i,t,0})^2\right).
$$
Hence,
\begin{align}
    \mathcal E_t^2
    \le
    c_0\lambda^2\ell^2 q\delta_t,
    \label{eq:grad-error}
\end{align}

The next outer iteration warm-starts the inner solvers from the
previous inner outputs, namely
$y_{i,t+1}^0=y_{i,t}^K$ and $z_{i,t+1}^0=z_{i,t}^K$.  The solution
maps $w_{\lambda,i}^*(\cdot)$ and $w_i^*(\cdot)$ are
$O(L_s/\mu)=O(\kappa)$-Lipschitz under Assumption 1; hence, for some
universal constant $c_w>0$,
$$
\begin{aligned}
&\left\lVert
    w_{\lambda,i}^*(\phi_{t+1})
    -w_{\lambda,i}^*(\phi_t)
\right\rVert
+
\left\lVert
    w_i^*(\phi_{t+1})-w_i^*(\phi_t)
\right\rVert \\
&\qquad \le
c_w\kappa
\left\lVert \phi_{t+1}-\phi_t \right\rVert .
\end{aligned}
$$
Thus, by Young's inequality, for each $i$,
\begin{align*}
    &\left\lVert
        y_{i,t+1}^0-w_{\lambda,i}^*(\phi_{t+1})
    \right\rVert^2
    =
    \left\lVert
        y_{i,t}^K-w_{\lambda,i}^*(\phi_{t+1})
    \right\rVert^2        \\
    &\le \!
    2(E^y_{i,t,K})^2
    \!+\!2c_w^2\kappa^2
    \left\lVert \phi_{t+1}\!-\!\phi_t \right\rVert^2,
\end{align*}
and similarly,
\begin{align*}
    \left\lVert
        z_{i,t+1}^0-w_i^*(\phi_{t+1})
    \right\rVert^2
    &\!\le\!
    2(E^z_{i,t,K})^2
    \!+\!2c_w^2\kappa^2
    \left\lVert \phi_{t+1}\!-\!\phi_t \right\rVert^2 .
\end{align*}
Adding the last two bounds and averaging over $i=1,\ldots,m$ gives
\begin{align*}
    \delta_{t+1}
    &\le
    \frac{2}{m}\sum_{i=1}^m
    \left((E^y_{i,t,K})^2+(E^z_{i,t,K})^2\right)                  \\
    &\quad
    +4c_w^2\kappa^2\left\lVert \phi_{t+1}-\phi_t \right\rVert^2 .
\end{align*}
The contraction estimates for the two inner solvers imply
\begin{align*}
    &\frac1m\sum_{i=1}^m
    \left((E^y_{i,t,K})^2+(E^z_{i,t,K})^2\right)                  \\
    &\qquad\le
    \frac{q}{m}\sum_{i=1}^m
    \left((E^y_{i,t,0})^2+(E^z_{i,t,0})^2\right)
    =
    q\delta_t .
\end{align*}
Combining the two displays and absorbing the numerical constant
$4c_w^2$ into $c_\delta$ gives
\begin{align}
    \delta_{t+1}
    &\le
    2q\,\delta_t
    +c_\delta\kappa^2
    \left\lVert \phi_{t+1}-\phi_t \right\rVert^2 ,
    \label{eq:delta-one}
\end{align}
for a universal constant $c_\delta>0$.  This constant is independent
of $c_0$; it only absorbs the constants from the Lipschitz bounds for
the solution maps and Young's inequality.  Choose $K$ large enough so
that $2q\le 1/2$.  Then
\begin{align}
    \delta_t
    \le
    2^{-t}\delta_0
    +c_\delta\kappa^2
    \sum_{j=0}^{t-1}2^{-(t-1-j)}
        \left\lVert \phi_{j+1}-\phi_j \right\rVert^2 .
    \label{eq:delta-telescope}
\end{align}

Let
$$
    \gamma:=c_0\lambda^2\ell^2q .
$$
Combining \eqref{eq:descent}, \eqref{eq:grad-error}, and
\eqref{eq:delta-telescope}, and telescoping over
$t=0,\ldots,N-1$, yields
\begin{align}
    \frac{\eta}{2}
    \sum_{t=0}^{N-1}
    \left\lVert\nabla\mathcal L_\lambda^*(\phi_t)\right\rVert^2
    &\le
    \mathcal L_\lambda^*(\phi_0)-\inf_\phi\mathcal L_\lambda^*(\phi)
    +O(\eta\gamma\delta_0)                                      \notag\\
    &
    -\!\left(\!\!
        \frac{1}{4\eta}
        \!-\!O(\eta\gamma\kappa^2)
    \!\!\right)
    \sum_{t=0}^{N-1}
    \left\lVert\phi_{t+1}-\phi_t\right\rVert^2 .
    \label{eq:main-telescope}
\end{align}
Since $\kappa=\ell/\mu$, the definition of $q$ gives
$$
    q=\exp\!\left(-\frac{K}{4\kappa}\right).
$$
Let $a_1,a_2>0$ denote sufficiently small universal constants.  The
constant $a_1$ is chosen so that
$O(\eta\gamma\kappa^2)\le 1/(4\eta)$ in
\eqref{eq:main-telescope}, which makes the displacement term
nonpositive.  The constant $a_2$ is chosen so that the accumulated
initial inner-solver error $O(\eta\gamma\delta_0)$ is absorbed into
the final numerator.  Thus it is enough to impose
$$
    \gamma=c_0\lambda^2\ell^2q
    \le
    \min\left\{
        \frac{a_1}{\eta^2\kappa^2},
        \frac{a_2}{\eta}
    \right\}
$$
is guaranteed if
$$
    q
    \le
    \min\left\{
        \frac{a_1}{c_0\eta^2\kappa^2\lambda^2\ell^2},
        \frac{a_2}{c_0\eta\lambda^2\ell^2}
    \right\}.
$$
Equivalently, it suffices to choose
$$
    K
    \ge
    4\kappa
    \log\!\left(
        \max\left\{
            e,
            \frac{c_0\eta^2\kappa^2\lambda^2\ell^2}{a_1},
            \frac{c_0\eta\lambda^2\ell^2}{a_2}
        \right\}
    \right).
$$
With $\eta\asymp(\ell\kappa^3)^{-1}$, the logarithm is absorbed by
the prescribed choice
$$
    K=O\!\left(\kappa\log(\lambda\kappa)\right),
$$
after enlarging the hidden constant.  Indeed, substituting
$\eta\asymp(\ell\kappa^3)^{-1}$ gives
$$
    \eta^2\kappa^2\lambda^2\ell^2
    \asymp
    \frac{\lambda^2}{\kappa^4},
    \qquad
    \eta\lambda^2\ell^2
    \asymp
    \frac{\lambda^2\ell}{\kappa^3}.
$$
Moreover, in the nontrivial regime $\epsilon\le1$, the prescribed choice
$\lambda\gtrsim \ell\kappa^3/\epsilon$ implies
$\ell\le \lambda\kappa$.  Hence both logarithmic terms above are
bounded by a universal multiple of $\log(\lambda\kappa)$.
Consequently, $K=O(\kappa\log(\lambda\kappa))$ is sufficient.  Hence
$\gamma$ satisfies the desired bound.  The last term in
\eqref{eq:main-telescope} is then nonpositive, and since
$\delta_0=O(R)$, while $O(\eta\gamma\delta_0)=O(R)$ by the choice of
$a_2$,
$$
    \frac1N
    \sum_{t=0}^{N-1}
        \left\lVert \nabla\mathcal L_\lambda^*(\phi_t)\right\rVert^2
    \le
    O\!\left(
        \frac{\Delta+R}{\eta N}
    \right).
$$
Taking
$$
    N=
    O\!\left(
        (\Delta+R)\ell\kappa^3\epsilon^{-2}
    \right)
$$
ensures that some iterate $\phi_{\hat t}$ satisfies
$$
    \left\lVert \nabla\mathcal L_\lambda^*(\phi_{\hat t}) \right\rVert\le c\epsilon
$$
for a universal constant $c>0$.  Proposition~\ref{pro:5}(a) and the choice
$\lambda\gtrsim \ell\kappa^3/\epsilon$ further imply
$$
    \begin{aligned}
    \left\lVert\nabla\mathcal L(\phi_{\hat t})\right\rVert
    \le
    &\left\lVert
        \nabla\mathcal L_\lambda^*(\phi_{\hat t})
    \right\rVert\\
    +
    &\left\lVert
        \nabla\mathcal L_\lambda^*(\phi_{\hat t})
        -\nabla\mathcal L(\phi_{\hat t})
    \right\rVert
    =
    O(\epsilon).
    \end{aligned}
$$
Rescaling the hidden numerical constants gives an $\epsilon$-first-order
stationary point of $\mathcal L$.

Each meta-training outer iteration uses
$$
    K+1
    =
    O\!\left(
        \kappa\log\!\left(\frac{\ell\kappa}{\epsilon}\right)
    \right)
$$
first-order oracle calls.  Multiplying by the outer iteration bound gives
$$
    O\!\left(
        (\Delta+R)\ell\kappa^4\epsilon^{-2}
        \log\!\left(\frac{\ell\kappa}{\epsilon}\right)
    \right)
$$
oracle calls.  This is the explicit version of the complexity bound.
Treating the fixed initial quantities $\Delta$ and $R$ as
problem-dependent constants, as in the reference proof, gives the
stated complexity.
\end{proof}

It is near-optimal in the sense that it matches the lower complexity bound $\Omega(\epsilon^{-2})$ in~\cite{carmon2020lower}.

\subsection{Dataset Split}
\label{app:dataset_split}

Following the meta-learning setting, the datasets were divided into meta-training and meta-testing classes. Specifically, 4 classes in Pap-Smear and ISIC2018 were used for meta-training and the remaining 3 classes were used for meta-testing. For BreaKHis40X, 5 classes were used for meta-training and the remaining 3 classes were used for meta-testing. The resulting meta-training and meta-testing splits were used to evaluate the few-shot generalization ability of the model on unseen classes.For the cross-domain experiments, BCHI was used only as the target domain, with all 3 histological grades used as meta-testing classes.

\subsection{Experiment Details}
\label{app:experiment_details}

For fair comparison, all methods were evaluated under the same experimental protocol, including the backbone, few-shot setting, and data augmentation strategy. Detailed descriptions of the experiment environment, network architecture, augmentation strategy, training/testing protocol, and hyperparameter settings.

All experiments were conducted on a server equipped with a 25-vCPU Intel Xeon Platinum 8470Q processor, 90GB of memory, and an NVIDIA RTX 5090 GPU with 32GB VRAM. The algorithm was implemented in Python 3.12 based on the PyTorch 2.8.0 deep learning framework, with CUDA 12.8 used for GPU acceleration, under the Ubuntu 22.04 operating system.

For fair comparison with baselines, we adopted a CNN4 backbone with an input size of 84×84, consisting of four 3×3 convolutional layers with 64 filters, each followed by batch normalization and ReLU activation. The random seed was set to 10, and all experiments were conducted under a 3-way setting with 3-shot, 5-shot, and 10-shot evaluation.During meta-training, we ran 5000 iterations with a task batch size of 32 and a regularization parameter of 0.5, and used the same baseline augmentation as Tra-MAML~\cite{voon2025trapezoidal,wei2019closer}, including random resized crop, image jitter, random vertical flip, and random horizontal flip, for fair comparison.During testing, the best weights obtained in meta-training were evaluated on 600 meta-testing tasks, with 30 adaptation steps performed for each task. 

The detailed hyperparameter settings of ABO-Med~(Algorithm~\ref{alg:ABO-Med}) used in the comparative experiments with other few-shot learning methods on the Pap-Smear, ISIC2018, and BreakHis40X datasets are summarized in Table~\ref{tab:hyperparameters}.
\begin{table}[H]
\footnotesize
\centering
\caption{Experimental hyperparameter settings for ABO-Med.}
\label{tab:hyperparameters}
\begin{tabular}{p{0.48\linewidth} p{0.42\linewidth}}
\toprule
\textbf{Hyperparameter} & \textbf{Value} \\
\midrule
Outer learning rate, $\eta$ & 1 \\
Inner learning rate for $z$ update, $\alpha$ & 0.005 \\
Inner learning rate for $y$ update, $\tau$ & 0.05 \\
Lagrangian multiplier, $\lambda$ & 1 \\
Number of inner iterations, $K$ & 30 \\
 number of adaptation steps $K'$ & 30 \\
Meta-testing learning rate, $\beta$ & 0.01 \\
Momentum parameter, $\theta$ & 0.9 \\
\bottomrule
\end{tabular}
\end{table}

The hyperparameters of ABO-Med were selected based on validation performance. Specifically, we performed a grid search over the main optimization parameters within predefined ranges and chose the final configuration according to the best validation results. The search ranges were set as follows: $\eta \in \{0.001, 0.01, 0.1, 0.5, 1, 5, 10, 20, 100\}$, $\alpha \in \{0.001, 0.005, 0.01\}$, $\tau \in \{0.01, 0.05, 0.1\}$, $\lambda \in \{0.1, 0.5, 1\}$, $K \in \{10, 20, 30, 40, 50\}$, $K' \in \{10, 20, 30, 40, 50\}$, $\beta \in \{0.001, 0.005, 0.01\}$, and $\theta \in \{0.5, 0.7, 0.9\}$. For ABO-Med with MedRAug, we used the same hyperparameter settings as ABO-Med to isolate the effect of the augmentation strategy. The final hyperparameter settings are reported in Table~\ref{tab:hyperparameters}.

\subsection{Ablation Study on Backbones}
\label{app:ablation_study_backbones}

To study the effect of different backbones, we compared CNN4, ResNetFeats10, ResNetFeats18, ResNetFeats25, and ViTFeatsTiny on three medical few-shot classification datasets. To ensure a fair comparison, all methods were evaluated using the same baseline data augmentation settings. All experiments were conducted under the 3-way 3-, 5-, and 10-shot settings, and the results are shown in Table~\ref{tab:backbone_ablation}. Although CNN4 has the fewest trainable parameters, it achieved the best overall performance. In contrast, deeper ResNet backbones and ViTFeatsTiny did not provide further improvements.

These results suggest that simply increasing backbone capacity does not necessarily lead to better performance in medical few-shot image classification. Unlike natural-image few-shot benchmarks, where categories often exhibit larger semantic differences and high-capacity backbones can benefit from rich visual diversity, medical image tasks are usually more fine-grained and rely on subtle morphological or texture differences. Under such limited-support settings, deeper and more complex models may be more prone to overfitting, or may require stronger regularization and larger annotated medical datasets to fully exploit their capacity~\cite{al2019reinforcing}. In contrast, the lightweight CNN4 backbone provides a more suitable balance between representation capacity and model complexity in our experiments, leading to more stable performance across different medical datasets and shot settings. 

\begin{table}[t]
\footnotesize
\centering
\caption{Ablation study of different feature backbones on the Pap-Smear, ISIC2018, and BreakHis40X datasets. The number of trainable parameters of the feature extractor is reported below each backbone name. Results are presented as mean test accuracy (\%) with 95\% confidence intervals, and the best results are shown in bold.}
\label{tab:backbone_ablation}
\renewcommand{\arraystretch}{1.12}
\setlength{\tabcolsep}{3.5pt}
\footnotesize
\begin{tabularx}{\columnwidth}{
>{\raggedright\arraybackslash}p{2.15cm}
>{\raggedright\arraybackslash}p{1.95cm}
>{\centering\arraybackslash}p{1.0cm}
>{\centering\arraybackslash}p{1.0cm}
>{\centering\arraybackslash}p{1.0cm}
}
\toprule
Backbone & Dataset & 3-shot & 5-shot & 10-shot \\
\midrule

\multirow{3}{*}{\shortstack[l]{CNN4\\{\scriptsize (0.113M)}}}
& Pap-Smear
& \textbf{81.51 $\pm$ 0.53}
& \textbf{86.78 $\pm$ 0.65}
& \textbf{87.62 $\pm$ 0.61} \\

& ISIC2018
& \textbf{53.56 $\pm$ 0.65}
& \textbf{58.28 $\pm$ 0.52}
& \textbf{63.67 $\pm$ 0.64} \\

& BreakHis40X
& \textbf{55.37 $\pm$ 0.83}
& \textbf{69.07 $\pm$ 0.79}
& \textbf{76.17 $\pm$ 0.89} \\
\midrule

\multirow{3}{*}{\shortstack[l]{ResNetFeats10\\{\scriptsize (4.898M)}}}
& Pap-Smear
& 75.57 $\pm$ 0.68
& 80.69 $\pm$ 0.55
& 84.18 $\pm$ 0.52\\

& ISIC2018
& 52.81 $\pm$ 0.64
& 55.59 $\pm$ 0.61
& 61.68 $\pm$ 0.63\\

& BreakHis40X
& 53.46 $\pm$ 0.63
& 65.92 $\pm$ 0.57
& 73.65 $\pm$ 0.54 \\
\midrule

\multirow{3}{*}{\shortstack[l]{ResNetFeats18\\{\scriptsize (11.169M)}}}
& Pap-Smear
& 72.87 $\pm$ 0.66
& 77.42 $\pm$ 0.62
& 81.19 $\pm$ 0.58 \\

& ISIC2018
& 51.94 $\pm$ 0.69
& 52.23 $\pm$ 0.63
& 57.48 $\pm$ 0.59\\

& BreakHis40X
& 52.48 $\pm$ 0.58
& 61.07 $\pm$ 0.52
& 68.42 $\pm$ 0.45\\
\midrule

\multirow{3}{*}{\shortstack[l]{ResNetFeats25\\{\scriptsize (17.440M)}}}
& Pap-Smear
& 71.11 $\pm$ 0.64
& 75.69 $\pm$ 0.62
& 78.48 $\pm$ 0.55 \\

& ISIC2018
& 48.56 $\pm$ 0.63
& 50.39 $\pm$ 0.65
& 54.25 $\pm$ 0.57\\

& BreakHis40X
& 49.13 $\pm$ 0.69
& 56.16 $\pm$ 0.46
& 63.86 $\pm$ 0.66 \\
\midrule

\multirow{3}{*}{\shortstack[l]{ViTFeatsTiny\\{\scriptsize (5.524M)}}}
& Pap-Smear
& 74.43 $\pm$ 0.73
& 78.64 $\pm$ 0.67
& 80.77 $\pm$ 0.64\\

& ISIC2018
&  51.39 $\pm$ 0.68
&  54.74 $\pm$ 0.63
&  60.23 $\pm$ 0.57\\

& BreakHis40X
& 52.91 $\pm$ 0.76
& 63.44 $\pm$ 0.71
& 70.28 $\pm$ 0.65\\

\bottomrule
\end{tabularx}
\end{table}

\subsection{Ablation Study on Data Augmentation}

\begin{figure*}[h]
    \centering
    \caption{Examples of different data augmentation strategies on the Pap-Smear, ISIC2018, and BreakHis datasets.}
    \includegraphics[width=1\linewidth]{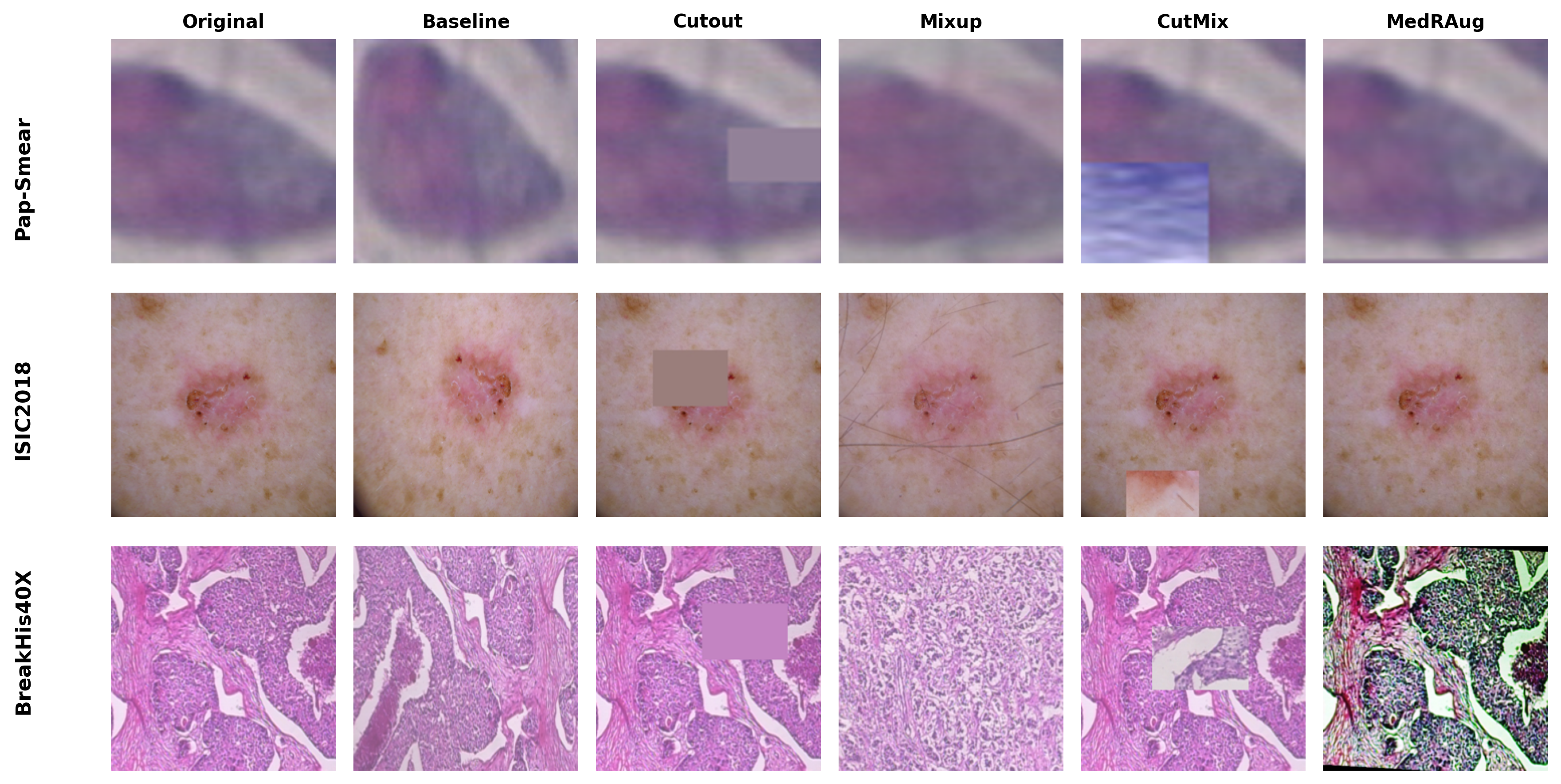}
    \label{fig:augmentation_examples}
\end{figure*}

\subsubsection{Details of Dataset-specific Augmentation Spaces}
\label{app:augmentation_space}

This section provides the detailed augmentation operation spaces adopted by MedRAug for different datasets. While the main text presents the overall framework and augmentation strategy, the tables in this appendix further summarize the selected augmentation operations and the rationale behind their design for each medical imaging dataset. Specifically, the augmentation space for the Pap-smear dataset is designed to introduce moderate variations in microscopic imaging while preserving cellular morphological characteristics. The augmentation space for the ISIC2018 dataset focuses on variations in lesion appearance, imaging angle, and illumination conditions. In contrast, the augmentation space for the BreakHis40X dataset places greater emphasis on preserving tissue structural information while simulating staining differences and imaging variations introduced during the scanning process. The detailed augmentation operation spaces for these datasets are summarized in Tables \ref{tab:pap_aug_space}, \ref{tab:isic_aug_space}, and \ref{tab:breakhis_aug_space}, respectively.

\begin{table}[t]
\centering
\caption{Augmentation operation space for Pap-smear cytology images.}
\label{tab:pap_aug_space}
\renewcommand{\arraystretch}{1.12}
\setlength{\tabcolsep}{3pt}
\scriptsize
\begin{tabularx}{\columnwidth}{
@{}
>{\raggedright\arraybackslash}p{1.35cm}
>{\raggedright\arraybackslash}p{2.05cm}
>{\raggedright\arraybackslash}X
@{}}
\toprule
\textbf{Operation} & \textbf{Type} & \textbf{Rationale} \\
\midrule
Identity & Baseline operation & Keeps the original cell image unchanged and avoids unnecessary structural perturbation. \\
Rotate & Geometric transformation & Simulates cell orientation variation, since cytology images usually have no fixed directional semantics. \\
Scale & Geometric transformation & Simulates differences in cell size and microscopic magnification. \\
Brightness & Photometric transformation & Simulates illumination and exposure variation in microscopic imaging. \\
Contrast & Photometric transformation & Simulates contrast differences caused by staining and imaging conditions. \\
RGBShift & Color transformation & Applies mild channel-wise color perturbation to simulate staining and color response variation. \\
Gamma & Intensity transformation & Simulates nonlinear brightness response caused by microscope imaging and exposure differences. \\
\bottomrule
\end{tabularx}
\end{table}

\begin{table}[t]
\centering
\caption{Augmentation operation space for ISIC dermoscopic images.}
\label{tab:isic_aug_space}
\renewcommand{\arraystretch}{1.12}
\setlength{\tabcolsep}{3pt}
\scriptsize
\begin{tabularx}{\columnwidth}{
@{}
>{\raggedright\arraybackslash}p{1.35cm}
>{\raggedright\arraybackslash}p{2.05cm}
>{\raggedright\arraybackslash}X
@{}}
\toprule
\textbf{Operation} & \textbf{Type} & \textbf{Rationale} \\
\midrule
Identity & Baseline operation & Preserves the original dermoscopic image and avoids over-augmentation. \\
Rotate & Geometric transformation & Simulates lesion orientation variation caused by camera angle or patient positioning. \\
Translate-X & Geometric transformation & Simulates horizontal lesion position variation within the image. \\
Translate-Y & Geometric transformation & Simulates vertical lesion position variation while preserving lesion appearance. \\
Brightness & Photometric transformation & Simulates illumination differences under different imaging conditions. \\
Contrast & Photometric transformation & Improves robustness to contrast variation between lesion and surrounding skin. \\
Color & Color transformation & Simulates skin tone, device, and acquisition-related color differences. \\
Sharpness & Texture-related transformation & Adjusts local texture clarity and improves robustness to image focus variation. \\
\bottomrule
\end{tabularx}
\end{table}

\begin{table}[t]
\centering
\caption{Augmentation operation space for BreakHis pathological images.}
\label{tab:breakhis_aug_space}
\renewcommand{\arraystretch}{1.12}
\setlength{\tabcolsep}{3pt}
\scriptsize
\begin{tabularx}{\columnwidth}{
@{}
>{\raggedright\arraybackslash}p{1.35cm}
>{\raggedright\arraybackslash}p{2.05cm}
>{\raggedright\arraybackslash}X
@{}}
\toprule
\textbf{Operation} & \textbf{Type} & \textbf{Rationale} \\
\midrule
Identity & Baseline operation & Keeps the original image unchanged and prevents excessive augmentation. \\
Rotate & Geometric transformation & Simulates tissue orientation variation caused by slide preparation and scanning. \\
Translate-X & Geometric transformation & Simulates horizontal position shifts of tissue regions. \\
Translate-Y & Geometric transformation & Simulates vertical position shifts of tissue regions. \\
Shear-X & Geometric transformation & Introduces mild horizontal shape variation while preserving tissue morphology. \\
Shear-Y & Geometric transformation & Introduces mild vertical shape variation while maintaining global tissue structure. \\
Brightness & Photometric transformation & Simulates illumination and scanner exposure variation. \\
Contrast & Photometric transformation & Improves robustness to contrast variation caused by staining and scanning conditions. \\
Color & Color transformation & Simulates color distribution shifts caused by staining and acquisition differences. \\
Sharpness & Texture-related transformation & Adjusts local edge and texture clarity to improve robustness to focus variation. \\
Equalize & Histogram transformation & Normalizes global intensity distribution and reduces uneven contrast effects. \\
HSVShift & Color-space transformation & Perturbs hue and saturation to simulate general color variability in RGB pathology images. \\
HEDShift & Stain-space transformation & Perturbs hematoxylin and eosin-related stain components to simulate H\&E staining variation. \\
\bottomrule
\end{tabularx}
\end{table}

\subsubsection{Experimental results on ablation study}
\label{app:aug_ablation_results}

To evaluate the impact of different data augmentation strategies on few-shot medical image classification and examine the sensitivity of diagnosis-related visual patterns to augmentation design, we conduct an ablation study. Since inappropriate augmentations may distort critical lesion structures and hinder feature learning, only the augmentation strategy is changed while all other settings remain unchanged. For a fair comparison, all methods use the same backbone network. Specifically, we compare the baseline augmentation strategy with MedRAug, CutMix, Mixup, and Cutout. The results are shown in Table \ref{tab:aug_ablation_results}.

Overall, the results show that conventional augmentation methods exhibit dataset-dependent effects. CutMix, Mixup, and Cutout improve performance in some settings on BreakHis40X, but they generally degrade performance on Pap-Smear and ISIC2018, suggesting that direct image mixing or random masking may disrupt important medical visual patterns. In contrast, MedRAug consistently achieves the best performance across all datasets and shot settings. Compared with the baseline augmentation, MedRAug improves the average accuracy by 1.69\%, 2.72\%, and 6.01\% on Pap-Smear, ISIC2018, and BreakHis40X, respectively. These results indicate that modality-aware augmentation is more effective than directly applying general-purpose augmentation strategies to few-shot medical image classification.

\begin{table}[t]
\footnotesize
\centering
\caption{Ablation study of baseline data augmentation strategies, Cutmix, Mixup, Cutout, and MedRAug, on the Pap-Smear, ISIC2018, and BreakHis40X datasets. The results are reported as mean test accuracy (\%) with their 95\% confidence intervals. }
\label{tab:aug_ablation_results}
\renewcommand{\arraystretch}{1.15}
\setlength{\tabcolsep}{6pt}
\footnotesize
\begin{tabularx}{\columnwidth}{
>{\raggedright\arraybackslash}p{1.5cm}
>{\raggedright\arraybackslash}p{1.5cm}
>{\centering\arraybackslash}p{1.0cm}
>{\centering\arraybackslash}p{1.0cm}
>{\centering\arraybackslash}p{1.0cm}
}
\toprule
Dataset & Method & 1-shot & 5-shot & 10-shot \\
\midrule

\multirow{4}{*}{Pap-Smear}
& Baseline
& 81.51 $\pm$ 0.68
& 86.78 $\pm$ 0.65
& 87.62 $\pm$ 0.61 \\
& Cutmix
& 71.37 $\pm$ 0.66
& 74.58 $\pm$ 0.65
& 80.11 $\pm$ 0.54\\
& Mixup
& 70.56 $\pm$ 0.62
& 72.39 $\pm$ 0.58
& 79.46 $\pm$ 0.51\\
& Cutout
& 73.67 $\pm$ 0.69
& 80.47 $\pm$ 0.64
& 82.56 $\pm$ 0.62 \\
& MedRAug
& \textbf{83.11 $\pm$ 0.67}
& \textbf{87.42 $\pm$ 0.65}
& \textbf{90.45 $\pm$ 0.58} \\
\midrule

\multirow{4}{*}{ISIC2018}
& Baseline
& 53.56 $\pm$ 0.65
& 58.28 $\pm$ 0.62
& 63.67 $\pm$ 0.64 \\
& Cutmix
& 49.41 $\pm$ 0.67
& 50.67 $\pm$ 0.63
& 56.02 $\pm$ 0.58\\
& Mixup
& 47.57 $\pm$ 0.74
& 52.16 $\pm$ 0.69
& 53.79 $\pm$ 0.63 \\
& Cutout
& 52.26 $\pm$ 0.71
& 53.63 $\pm$ 0.67
& 57.31 $\pm$ 0.66 \\
& MedRAug
& \textbf{55.24 $\pm$ 0.72}
& \textbf{61.42 $\pm$ 0.65}
& \textbf{67.02 $\pm$ 0.58} \\
\midrule
\multirow{4}{*}{BreakHis40X}
& Baseline
& 55.37 $\pm$ 0.83
& 69.07 $\pm$ 0.79
& 76.17 $\pm$ 0.69 \\
& Cutmix
& 61.22 $\pm$ 0.85
& 70.19 $\pm$ 0.81
& 74.18 $\pm$ 0.78 \\
& Mixup
& 62.44 $\pm$ 0.77
& 68.73 $\pm$ 0.73
& 73.71 $\pm$ 0.68 \\
& Cutout
& 62.26 $\pm$ 0.86
& 69.83 $\pm$ 0.82
& 73.16 $\pm$ 0.76 \\
& MedRAug
& \textbf{65.81 $\pm$ 0.89}
& \textbf{72.61 $\pm$ 0.84}
& \textbf{80.21 $\pm$ 0.78} \\
\bottomrule
\end{tabularx}
\end{table}

\subsubsection{Data Augmentation Strategy}
\label{app:data_augment_strategy}

\paragraph{Cutout}
Cutout is a simple yet effective data augmentation technique that improves model robustness by randomly masking out a rectangular region of the input image during training. By introducing local occlusions, Cutout prevents the model from relying excessively on a few discriminative regions and encourages it to learn more distributed and robust feature representations. This strategy can alleviate overfitting and improve generalization, especially in limited-data settings. However, for medical image analysis, the size and position of the masked region should be carefully controlled, since excessive occlusion may remove diagnostically important structures.

\paragraph{Mixup}
Mixup is a data augmentation technique based on linear interpolation, which generates virtual training samples by combining two input samples and their corresponding labels with a weighted sum, thereby enriching the training distribution and improving model generalization. Specifically, for any two samples $(x_i, y_i)$ and $(x_j, y_j)$, the mixed sample is constructed as
\[
\tilde{x}=\rho x_i+(1-\rho)x_j,\qquad
\tilde{y}=\rho y_i+(1-\rho)y_j,
\]
where $\rho \in [0,1]$ is the mixing coefficient, usually sampled from a Beta distribution:
\[
\rho \sim \mathrm{Beta}(a,a).
\]
By introducing interpolated samples into the training space, Mixup encourages the model to learn smoother decision boundaries and reduces overfitting to limited training data. In medical image analysis, this strategy can improve robustness, although excessively strong mixing may weaken diagnostically important structures, making proper selection of $a$ necessary.

\paragraph{CutMix}
CutMix is a data augmentation method that combines regional cropping with sample mixing. Unlike Mixup, which performs linear interpolation over the entire image, CutMix constructs a new sample by replacing a local rectangular region of one image with the corresponding region from another image, thereby preserving more local structural information while introducing cross-class discriminative features. Specifically, for two training samples $(x_i, y_i)$ and $(x_i, y_i)$, the augmented sample is defined as
\[
\tilde{x}=M\odot x_i+(1-M)\odot x_i,\qquad
\tilde{y}=\rho y_i+(1-\rho)y_i,
\]
where $M\in\{0,1\}^{W\times H}$ denotes a binary mask, $W$ and $H$ are the width and height of the original image, and $\odot$ represents element-wise multiplication. The label mixing coefficient $\rho$ is determined by the area ratio of the replaced region:
\[
\rho=1-\frac{wh}{WH},
\]
where $w$ and $h$ are the width and height of the replaced patch, and $WH$ denotes the total area of the original image. In this way, CutMix improves robustness to local perturbations and encourages the model to learn more effective discriminative representations. 

\newpage

For intuitive comparison, Figure~\ref{fig:augmentation_examples} shows representative examples from the Pap-Smear, ISIC2018, and BreakHis datasets under the original image, baseline data augmentation, Cutout, Mixup, CutMix and MedRAug. The figure visually demonstrates the different transformation patterns introduced by these augmentation strategies.


\end{document}

%% file: math.tex
\usepackage{amsmath}\allowdisplaybreaks
\usepackage{amsfonts,bm}

\def\1{\bf{1}}

\DeclareMathOperator*{\argmin}{arg\,min}

\usepackage{amsthm}
\usepackage{etoolbox}
\theoremstyle{plain}

\makeatletter
\def\Ddots{\mathinner{\mkern1mu\raise\p@
\vbox{\kern7\p@\hbox{.}}\mkern2mu
\raise4\p@\hbox{.}\mkern2mu\raise7\p@\hbox{.}\mkern1mu}}
\makeatother

\makeatletter
\newcommand*{\rom}[1]{\expandafter\@slowromancap\romannumeral #1@}
\makeatother

\newtheorem{theorem}{Theorem}
\newtheorem{proposition}[theorem]{Proposition}
\newtheorem{lemma}[theorem]{Lemma}

\theoremstyle{definition}
\newtheorem{definition}[theorem]{Definition}
\newtheorem{assumption}[theorem]{Assumption}
\theoremstyle{remark}

\def\0{{\bf 0}}
\def\1{{\bf 1}}

\usepackage{enumitem}
\usepackage{hhline}